%% file: ms.tex
\documentclass[letterpaper]{article} 
\usepackage{aaai23}  
\usepackage{times}  
\usepackage{helvet}  
\usepackage{courier}  
\usepackage[hyphens]{url}  
\usepackage{graphicx} 
\usepackage{natbib}  
\usepackage{caption} 
\usepackage{algorithm}
\usepackage[noend]{algorithmic}

\usepackage{newfloat}
\usepackage{listings}
\DeclareCaptionStyle{ruled}{labelfont=normalfont,labelsep=colon,strut=off} 
\floatstyle{ruled}
\newfloat{listing}{tb}{lst}{}
\floatname{listing}{Listing}
\usepackage{booktabs}
\usepackage{amsfonts}
\usepackage{amssymb}
\usepackage{amsmath}
\usepackage{amsthm}
\usepackage{esint}
\usepackage{isomath}
\usepackage{latexsym}
\usepackage{epic}
\usepackage{epsfig}
\usepackage{multirow}
\usepackage{hhline}
\usepackage{wrapfig}
\usepackage{verbatim}
\usepackage{enumitem}
\usepackage{color}
\usepackage{thmtools}
\usepackage{thm-restate}
\usepackage{bm}
\usepackage{dsfont}

\usepackage[labelformat=simple]{subcaption}
\usepackage{mathtools}
\usepackage{graphicx}
\usepackage{tabu}
\usepackage{colortbl}
\usepackage[dvipsnames]{xcolor}         
\usepackage{natbib}
\renewcommand{\cite}{\citep}

\usepackage{tikz}
\usepackage{pgfplots}
\usepackage{pgfplotstable}
\usepgfplotslibrary{fillbetween}
\pgfplotsset{compat=1.16}
\pgfdeclarelayer{ft}
\pgfdeclarelayer{bg}
\pgfsetlayers{bg,main,ft}

\newcommand{\optimistic}{\dagger}

\newcommand{\N}{\mathbb{N}}

\newcommand{\R}{\mathbb{R}}
\renewcommand{\epsilon}{\varepsilon}

\newcommand{\norm}[1]{\left\lVert#1\right\rVert}

\newcommand{\whittle}{m}
\newcommand{\transition}{P}
\newcommand{\transitions}{\bm{P}}
\newcommand{\truep}{\transitions^\star}

\newcommand{\state}{s}
\newcommand{\statevec}{\bm{s}}
\newcommand{\nstates}{M}
\newcommand{\stateset}{\mathcal{S}} 
\newcommand{\action}{a}
\newcommand{\actionvec}{\bm{a}}
\newcommand{\Action}{\mathcal{A}}

\newcommand{\policy}{\pi}
\newcommand{\optpolicy}{\policy^\star}

\newcommand{\dataset}{\mathcal{D}}

\newcommand{\maxPenalty}{\mathcal{P}_m}
\newcommand{\maxValue}{\mathcal{P}_V}

\newcommand{\round}{t}
\newcommand{\Round}{T}
\newcommand{\horizon}{h}
\newcommand{\Horizon}{H}

\theoremstyle{plain}
\newtheorem{theorem}{Theorem}[section]

\newtheorem{lemma}[theorem]{Lemma}

\theoremstyle{definition}
\newtheorem{definition}[theorem]{Definition}
\newtheorem{assumption}[theorem]{Assumption}

\theoremstyle{remark}

\title{Optimistic Whittle Index Policy: Online Learning for Restless Bandits}

\author{
    Kai Wang\equalcontrib,\thanks{Work done during an internship at Google Research.}\textsuperscript{\rm 1}
    Lily Xu\equalcontrib,\footnotemark[2]\textsuperscript{\rm 1}
    Aparna Taneja,\textsuperscript{\rm 2}
    Milind Tambe\textsuperscript{\rm 1,2}
}
\affiliations{
    \textsuperscript{\rm 1}Harvard University\\
    \textsuperscript{\rm 2}Google Research

    \{kaiwang, lily\_xu\}@g.harvard.edu, \{aparnataneja, milindtambe\}@google.com
}

\begin{document}

\maketitle

\begin{abstract}
  Restless multi-armed bandits (RMABs) extend multi-armed bandits to allow for stateful arms, where the state of each arm evolves restlessly with different transitions depending on whether that arm is pulled. Solving RMABs requires information on transition dynamics, which are often unknown upfront. To plan in RMAB settings with unknown transitions, we propose the first online learning algorithm based on the Whittle index policy, using an upper confidence bound (UCB) approach to learn transition dynamics. Specifically, we estimate confidence bounds of the transition probabilities and formulate a bilinear program to compute optimistic Whittle indices using these estimates. Our algorithm, {\it UCWhittle}, achieves sublinear $O(H \sqrt{T \log T})$ frequentist regret to solve RMABs with unknown transitions in $T$ episodes with a constant horizon~$H$. Empirically, we demonstrate that UCWhittle leverages the structure of RMABs and the Whittle index policy solution to achieve better performance than existing online learning baselines across three domains, including one constructed from a real-world maternal and childcare dataset.
\end{abstract}

\input{body}

\bibliography{ref}

\clearpage
\appendix

\input{appendix}

\end{document}

%% file: body.tex
\section{Introduction}
Restless multi-armed bandits (RMABs)~\cite{whittle1988restless} generalize multi-armed bandits by introducing states for each arm. RMABs are commonly used to model sequential scheduling problems with limited resources such as in clinical health \cite{villar2015multi}, online advertising \cite{meshram2016optimal}, and energy-efficient scheduling \cite{borkar2017opportunistic}. 
As with stochastic combinatorial bandits \cite{chen2013combinatorial}, the RMAB learner must repeatedly pull $K$ out of $N$ arms at each timestep. Unlike stochastic bandits, the reward distribution of each arm in an RMAB depends on that
arm's state, which transitions based on a Markov decision process (MDP) depending on whether the arm is pulled. These problems are called ``restless'' as arms may change state regardless of whether they are pulled. The reward at each timestep is the sum of rewards across all arms, including arms not acted upon.

Even when the transition dynamics are given, planning an optimal policy for RMABs is PSPACE-hard~\cite{papadimitriou1994complexity} due to the state-dependent reward and combinatorial action space. To compute an approximate planning solution to RMABs, the \emph{Whittle index policy}~\cite{whittle1988restless} defines a ``Whittle index'' for each arm as an estimate of the future value if acted upon, then acts on the arms with the $K$ largest indices. The Whittle index policy is shown to be asymptotically optimal~\cite{weber1990index} and is commonly adopted as a scalable solution to RMAB problems~\cite{hsu2018age,kadota2016minimizing}. 

However, in many real-world applications of RMABs, transition dynamics are often unknown in advance. The learner must strategically query arms to learn the underlying transition probabilities while simultaneously achieving high reward. Accordingly, in this paper we focus on the challenge of online learning in RMABs with unknown transitions. 
We focus on the Whittle index policy due to its scalability and consider a fixed-length episodic RMAB setting.

\paragraph{Main contributions}
We present \emph{UCWhittle}, an upper confidence bound (UCB) algorithm that uses the Whittle index policy to achieve the first sublinear frequentist regret guarantee for RMABs.
Our algorithm maintains confidence bounds for every transition probability across all arms based on prior observations.
Using these bounds, we define a bilinear program to solve for optimistic transition probabilities --- the transition probabilities that yield the highest future reward.
These optimistic transition probabilities enable us to compute an \emph{optimistic Whittle index} for each arm to inform a Whittle index policy.
Our UCWhittle algorithm leverages the structure of RMABs and the Whittle index solution to decompose the policy across individual arms, greatly reducing the computational cost of finding an optimistic solution compared to other UCB-based solutions~\cite{auer2006logarithmic,jaksh2010near}. 

Theoretically, we analyze the frequentist regret of UCWhittle. The \emph{frequentist regret} is the worst-case regret incurred from unknown transition dynamics; in contrast, the \emph{Bayesian regret} is the regret averaged over all possible transitions from a prior distribution.
In this paper, we define \emph{regret} in terms of the relaxed Lagrangian of the RMAB --- to make the objective tractable --- which upper bounds the primal RMAB problem.
We show that UCWhittle achieves sublinear frequentist regret $O(H \sqrt{T \log T})$ where $T$ is the number of episodes of interaction with the RMAB instance and $H$ is a sufficiently large per-episode time horizon. 
Our result extends the analysis of Bayesian regret in RMABs~\cite{jung2019regret} to frequentist regret by removing the need to assume a prior distribution.
Finally, we evaluate UCWhittle against other online RMAB approaches on real maternal and child healthcare data~\cite{mate2022field} and two synthetic settings, showing that UCWhittle achieves lower frequentist regret empirically as well. 

\section{Background}

\paragraph{Offline planning for RMABs}
When the transition dynamics are given, an RMAB is an optimization problem in a sequential setting. Computing the optimal policy in RMABs is PSPACE-hard~\cite{papadimitriou1994complexity} due to the state-dependent reward distribution and combinatorial action space. The Whittle index policy~\cite{whittle1988restless} approximately solves the planning problem by estimating the value of each arm state. The indexability condition~\cite{akbarzadeh2019restless,wang2019opportunistic} guarantees asymptotic optimality~\cite{weber1990index} of the Whittle index policy with an infinite time horizon.
\citet{nakhleh2021neurwin} use deep reinforcement learning to estimate Whittle indices for episodic finite-horizon RMABs, which requires the environment to be differentiable and transitions known. 

\paragraph{Online learning for RMABs}
When the transition dynamics are unknown, an RMAB becomes an online learning problem in which the learner must simultaneously learn the transition probabilities (exploration) and execute high-reward actions (exploitation), with the objective of minimizing regret with respect to a chosen benchmark. \citet{dai2011non} achieve a regret bound of $O(\log T)$ benchmarked against an optimal policy from a finite number of potential policies. 
\citet{xiong2022reinforcement} use a Lagrangian relaxation and index-based algorithm, but require access to an offline simulator to generate samples for any given state--action pair. 
\citet{tekin2012online} define a weaker benchmark of the best single-action policy --- the optimal policy that continues to play the same arm --- and use a UCB-based algorithm to achieve $O(\log T)$ frequentist regret.


Recent works introduce oracle-based policies for the non-combinatorial setting in which the learner pulls a single arm in each round, receiving bandit feedback and observing only the state of the pulled arm. \citet{jung2019regret} use a Thompson sampling--based algorithm which achieves a Bayesian regret bound $O(\sqrt{T \log T})$ under a given prior distribution.
\citet{wang2020restless} use separate exploration and exploitation phases to achieve frequentist regret $O(T^{2/3})$. These works assume some policy oracle is given, thus benchmark regret with the policy given by the oracle with knowledge of the true transitions. In contrast to the meta-algorithms they propose, \emph{we design an optimal approach custom-tailored to one specific oracle --- based on the Whittle index policy --- which enables us to achieve a tighter frequentist regret bound of $O(H \sqrt{T \log T})$} with a constant horizon~$H$.



\paragraph{Online reinforcement learning}
RMABs are a special case of Markov decision processes (MDPs) with combinatorial state and action spaces.
Q-learning algorithms are popular for solving large MDPs and have been applied to standard binary-action RMABs \citep{avrachenkov2022whittle,fu2019towards,biswas2021learn} and extended to the multi-action setting \citep{killian2021q}.
However, these works do not provide regret guarantees. Significant work has explored online learning for stochastic multi-armed bandits \citep{neu2013efficient,immorlica2019adversarial,foster2020beyond,baek2020ts,xu2021dual}, but these do not allow arms to change state.

Some papers study online reinforcement learning by using the optimal policy as the benchmark to bound regret in MDPs \cite{auer2006logarithmic,jaksh2010near} and RMABs \cite{ortner2012regret}. These works use UCB-based algorithms (UCRL and UCRL2) to obtain a regret of $O(\sqrt{T \log T})$. 
However, evaluating regret with respect to the optimal policy requires computing the optimal solution to the RMAB problem, which is intractable due to the combinatorial space and action spaces. To overcome this difficulty, we restrict the benchmark for computing regret to the class of Whittle index threshold policies, and leverage the weak decomposability of the Whittle index threshold policy to establish a new regret bound.

\paragraph{Frequentist versus Bayesian regret}

The regret definition that we consider is \emph{frequentist} regret, measuring worst-case regret under unknown transition probabilities. The other regret notion is Bayesian regret: the expected regret over a prior distribution over possible transition functions. Bayesian regret, such as from Thompson sampling--based methods, relies on a prior and does not provide worst-case guarantees \cite{jung2019regret,jung2019thompson}.

\section{Restless Bandits and Whittle Index Policy}
An instance of a restless multi-armed bandit problem is composed of a set of $N$ arms. Each arm $i \in [N]$ is modeled as an independent Markov decision process (MDP) defined by a tuple $(\stateset, \Action, R, \transition_i)$. The state space~$\stateset$, action space~$\Action$, and reward function $R: \stateset \times \Action \rightarrow \R$ are shared across arms; the transition probability $\transition_i: \stateset \times \Action \times \stateset \rightarrow [0,1]$ may be unique per arm~$i$.

We denote the state of the RMAB instance at timestep $\horizon \in \N$ by $\statevec_{\horizon} \in \stateset^N$, where $s_{\horizon,i}$ denotes the state of arm~$i \in [N]$. We assume the state is fully observable.
The initial state is given by $\statevec_1 = \statevec_\text{init} \in \stateset^N$.
The action (a set of ``arm pulls'') at time $\horizon$ is denoted by a binary vector $\actionvec_\horizon \in \Action^N = \{0,1\}^N$ and is constrained by budget~$K$ such that $\sum\nolimits_{i \in [N]} \action_{\horizon, i} \leq K$.

After taking action~$a_{\horizon,i}$ on arm~$i$, the state $s_{\horizon,i}$ transitions to the next state~$s_{\horizon+1,i}$ with transition probability $\transition_i(s_{\horizon,i}, a_{\horizon,i}, s_{\horizon+1,i}) \in [0,1]$. We denote the set of all transition probabilities by $\transitions = [\transition_i]_{i \in [N]}$.
The learner receives reward $R(s_{\horizon,i}, a_{\horizon,i})$ from each arm~$i$ (including those not acted upon) at every timestep~$\horizon$; we assume the reward function~$R$ is known.

The learner's actions are described by a deterministic policy $\policy: \stateset^N \rightarrow \Action^N$ which maps a given state $\statevec \in \stateset^N$ to an action $\actionvec \in \Action^N$.
The learner's goal is to optimize the total discounted reward, with discount factor $\gamma \in (0, 1)$:
\begin{align}
    \max\limits_{\policy}  & \quad \mathop{\mathbb{E}}\limits_{(\statevec, \actionvec) \sim (\transitions, \policy)} \sum\nolimits_{\horizon \in \N} \gamma^{\horizon-1} \sum\nolimits_{i \in [N]} R(s_{\horizon,i}, a_{\horizon,i}) \nonumber \\
    \text{s.t.} & \quad \sum\nolimits_{i \in [N]} (\policy(\statevec))_i \leq K \quad \forall \statevec \in \stateset^N  
 \label{eqn:optimization-problem}
\end{align}
where $\statevec \sim \transitions$ indicates $s_{h,i} \sim \transition_{i}(\cdot \mid s_{h-1,i}, \policy_i(\statevec_{h-1}))$ and $\actionvec \sim \policy$ indicates $a_i \sim \policy_i(\statevec)$.


\subsection{Lagrangian Relaxation}
Equation~\ref{eqn:optimization-problem} is intractable to evaluate over all possible policies, thus a poor candidate objective for evaluating online learning performance.
Instead, we relax the constraints to use the Lagrangian as the evaluation metric:
\begin{align}
   & U^{\transitions, \lambda}_{\policy}(\statevec_{1}) \coloneqq \mathop{\mathbb{E}}_{(\statevec, \actionvec) \sim (\transitions, \policy)} \sum\nolimits_{\horizon \in \N}  \nonumber \\
   & \quad \gamma^{\horizon-1} \Biggl( \sum\limits_{i \in [N]} R(s_{\horizon,i}, a_{\horizon,i}) - \lambda \biggl( \sum\limits_{i \in [N]} (\policy(\statevec_{\horizon}))_i - K \biggr) \Biggr) \label{eqn:relaxed-lagrangian}
\end{align}
which also considers actions that exceed the budget constraint, subject to a given penalty~$\lambda$.
The optimal value of Equation~\ref{eqn:relaxed-lagrangian}, which we denote $U^{\transitions, \lambda}_\star$, is always an upper bound to Equation~\ref{eqn:optimization-problem}. Therefore, we solve Equation~\ref{eqn:relaxed-lagrangian} for candidate penalty values $\lambda$ and find the infimum $\lambda^\star = \arg\min_\lambda U^{\transitions, \lambda}_\star$ afterward.


\subsection{Whittle Index and Threshold Policy}
Relaxing the budget constraint enables us to decompose the combinatorial policy into a set of $N$ independent policies for each arm. The decoupled policy yields $\policy(\statevec) = [\policy_i(\statevec_i)]_{i \in [N]}$, where each arm policy $\policy_i: \stateset \rightarrow \Action$ specifies the action for arm~$i$ at state~$s_i$. The value function is then:
\begin{align}
    & V^{P_i, \lambda}_{\policy_i}(s_{1,i}) \coloneqq \mathop{\mathbb{E}}_{(s_{1,i}, a_{1,i}, s_{2,i}, a_{2,i}, \ldots) \sim (P_i, \policy_i)} \sum\nolimits_{\horizon \in \N} \nonumber \\
    & \qquad \gamma^{\horizon-1} \Biggl( R(s_{\horizon,i}, a_{\horizon,i}) - \lambda \biggl(\policy_i(s_{\horizon,i}) - K\biggr) \Biggr) \ . \label{eqn:decomposed-relaxed-lagrangian}
\end{align}

Equation~\ref{eqn:decomposed-relaxed-lagrangian} can be interpreted as adding a penalty $\lambda$ to the pulling action $a=1$, which motivates the definition of Whittle index~\cite{whittle1988restless} as the smallest penalty for an arm such that pulling that arm is as good as not pulling it:
\begin{definition}\label{def:whittle-index}
Given transition probabilities~$P_i$ and state~$s_i$, the \emph{Whittle index}~$W_i$ of arm~$i$ is defined as:
\begin{align}\label{eqn:whittle-index}
    W_i(P_i, s_i) &= \inf\limits_{m_i} \{ m_i : Q^{m_i}(s_i, 0) = Q^{m_i}(s_i, 1) \}
\end{align}
where the Q-function $Q^{m_i}(s_i,a_i)$ and value-function $V^{m_i}(s_i)$ are the solutions to the Bellman equation with penalty~$m_i$ for pulling action $a_i=1$: 
\begin{align}
    Q^{m_i}(s, a) &= -{m_i} a + R(s, a) + \gamma \sum\limits_{s' \in \stateset} P_i(s, a, s') V^{m_i}(s') \nonumber \\
    V^{m_i}(s) &= \max\limits_{a \in \Action} Q^{m_i}(s, a) \ . \nonumber
\end{align}
\end{definition}

When the Whittle index $W_i(P_i, s_i)$ for an arm is higher than the chosen global penalty $\lambda$ --- that is, $m_i > \lambda$ --- the optimal policy for Equation~\ref{eqn:decomposed-relaxed-lagrangian} is to pull that arm, i.e., $\policy_i(s_i)=1$.
We denote the Whittle indices of all arms and all states by $W(\transitions) = [ W_i(P_i, s_i) ]_{i \in [N], s_i \in \stateset} \in \R^{N \times |\stateset|}$.

\begin{definition}[Whittle index threshold policy]\label{def:wi-threshold-policy}
Given a chosen global penalty~$\lambda$ and the Whittle indices $W(\transitions)$ computed from transitions~$\transitions$, the threshold policy is defined by:
\begin{align}
    \policy_{W(\transitions), \lambda}(\statevec) = [\mathds{1}_{W_i(P_i,s_i) \geq \lambda }]_{i \in [N]} \in \Action^N  \ ,
\end{align}
which pulls all arms with Whittle indices larger than $\lambda$.
\end{definition}
The Whittle index threshold policy maximizes the relaxed Lagrangian in Equation~\ref{eqn:relaxed-lagrangian} under penalty $\lambda$, but may violate the budget constraints in Equation~\ref{eqn:optimization-problem}. In practice, we pull only the arms with the top $K$ Whittle indices to respect the strict budget constraint. 


\section{Problem Statement: \\ Online Learning in RMABs}
We consider the online setting where the true transition probabilities~$\truep$ are unknown to the learner. The learner interacts with an RMAB instance across multiple episodes, and only requires observations for the first $H$~timesteps of each episode to estimate transition probabilities.

At the beginning of each episode $\round \in [\Round]$, the learner starts the RMAB instance (timestep $h=1$) from $\boldsymbol \state_1 = \boldsymbol \state_{\text{init}}$ and selects a new policy~$\policy^{(\round)}$.
We consider the following setting: 
\begin{itemize}
    \item Each episode has an infinite horizon with discount factor~$\gamma$.
    \item In each episode~$\round$, the learner proposes a policy $\policy^{(\round)}$. The learner observes the first $H$ timesteps\footnote{In practice, infinite time horizon means a large horizon that is much larger than $H$.}, but receives the infinite discounted reward $U^{\transitions, \lambda}_{\policy^{(\round)}}(\statevec_1)$ to account for the long-term effect of $\policy^{(\round)}$. 
    \item We assume the MDP associated with each arm is \emph{ergodic}. That is, starting from the given initial state, we assume $H$ is large enough such that after $H$ timesteps, there is at least $\epsilon > 0$ probability of reaching any state $\statevec \in \stateset$.
    
\end{itemize}


To evaluate the performance of our policy $\policy^{(\round)}$, we compute \emph{regret} against a full-information benchmark: the Whittle index threshold policy~$\policy_{W(\truep), \lambda}$ with knowledge of the true transitions~$\truep$. This offline benchmark measures the advantage gained from knowing the true transitions~$\truep$.

\begin{definition}[Frequentist regret of the Lagrangian objective]\label{def:lagrangian-regret}
Given a penalty~$\lambda$ and the true transitions~$\truep$, we define the \emph{regret} of the policy~$\policy^{(\round)}$ in episode~$\round$ relative to the optimal policy $\optpolicy = \policy_{W(\truep), \lambda}$:
\begin{align}\label{eqn:individual-regret}
    & \text{Reg}^{(\round)}_\lambda \coloneqq U^{\truep, \lambda}_{\optpolicy}(\statevec_1) - U^{\truep, \lambda}_{\policy^{(\round)}}(\statevec_1) \ , \nonumber \\
    & \text{Reg}_\lambda(T) \coloneqq \sum\nolimits_{\round \in [T]}
    \text{Reg}^{(\round)}_\lambda  \ .
\end{align}
\end{definition}
However, the relaxed Lagrangian in Equation~\ref{eqn:relaxed-lagrangian} with a randomly chosen penalty $\lambda$ may not be a good proxy to the primal RMAB problem in Equation~\ref{eqn:optimization-problem}. Therefore, we define the Lagrangian using the optimal Lagrangian multiplier $\lambda^\star$ as the tightest upper bound of Equation~\ref{eqn:optimization-problem}.


\begin{definition}[Frequentist regret of the optimal Lagrangian objective]\label{def:optimal-lagrangian-regret}
    Given $\truep$, we denote the optimal penalty by $\lambda^\star = \arg\min_\lambda U^{\truep, \lambda}_{\optpolicy}(\statevec_1)$. The regret of the \emph{optimal Lagrangian objective} is defined by:
    \begin{align}
        & \text{Reg}^{(\round)}_{\lambda^\star} \coloneqq U^{\truep, \lambda^\star}_{\optpolicy}(\statevec_1) - U^{\truep, \lambda^\star}_{\policy^{(\round)}}(\statevec_1) \ , \nonumber \\ 
        & \text{Reg}_{\lambda^\star}(T) \coloneqq \sum\nolimits_{\round \in [T]} \text{Reg}^{(\round)}_{\lambda^\star}  \ .
    \end{align}
\end{definition}
The expected regret is approximated using the regret from the relaxed Lagrangian in Equation~\ref{eqn:relaxed-lagrangian} as defined in Definition~\ref{def:lagrangian-regret} and Definition~\ref{def:optimal-lagrangian-regret}.

\section{UCWhittle: Optimistic Whittle Index Threshold Policy}\label{sec:algorithm}
A key challenge to UCB-based online learning in RMABs is that the estimated transitions impact estimates of future reward, so optimistic estimates of transition probabilities do not correspond to optimistic estimates of reward. 
We introduce a method, UCWhittle, to compute optimistic Whittle indices that account for highest future value.

\subsection{Confidence Bounds of Transition Probabilities}
To compute confidence bounds for every unknown transition probability in the RMAB instance, we maintain counts $N_i^{(\round)}(s, a, s')$ for every state, action, and next state transition observed by episode $\round$. 


Given a chosen small constant $\delta > 0$, we estimate each transition probability $P_i(s, a, s')$ with the empirical mean
\newcommand{\conf}{\sqrt{\frac{2 |\stateset| \log(2 |\stateset| |\Action| N \frac{\round^4}{\delta})}{\max \{ 1, N_i^{(\round)}(s,a) \} }}}
\begin{align}
    \hat{P}^{(\round)}_i(s, a, s') \coloneqq \frac{N^{(\round)}_i(s, a, s')}{N^{(\round)}_i(s,a)}
\end{align}
and confidence radius
\begin{align}
    d^{(\round)}_i(s,a) \coloneqq \conf  \label{eqn:confidence-radius}
\end{align}
where $N_i^{(\round)}(s, a) \coloneqq \sum\nolimits_{s' \in \stateset} N^{(\round)}_i(s, a, s')$.
With these confidence bounds, the ball~$\boldsymbol B$ of possible values for transition probabilities~$\transitions$ is
\begin{align}
    \boldsymbol B^{(\round)} \! = \! \left\{ \transitions \! \mid \! \norm{P_i(s, a, \cdot) \! - \! \hat{P}^{(\round)}_i(s, a, \cdot)}_1 \! \leq \! d^{(\round)}_i(s,a) ~\forall i,s,a \right \}. \nonumber
\end{align}




\subsection{Optimistic Transitions and Whittle Indices}
To translate confidence bounds in transition probabilities to the actual reward, we define an optimization problem~\eqref{eqn:optimistic-whittle} to find for each arm~$i$ the \emph{optimistic} transition probability~$P_i^{\optimistic}$, the value within the confidence bound that yields the \textit{highest future value} from the starting state~$s_i$:
\begin{align}\label{eqn:optimistic-whittle}
    & \max_{V, Q, P_i \in \boldsymbol B_i^{(\round)}} \  V(s_i) \qquad \text{s.t.~~} V(s) = \max\limits_{a \in \Action} Q(s,a) \tag{$\maxValue$} \\
    & Q(s,a) = -\lambda a + R(s,a) + \gamma \sum\nolimits_{s' \in \stateset} P_i(s,a,s') V(s') \nonumber 
\end{align}
We prove Equation~\eqref{eqn:optimistic-whittle} to be optimal in Section~\ref{sec:regret}.

We use the optimistic transition $P_i^{\optimistic}$ to compute the corresponding \emph{optimistic Whittle index} $W^{\optimistic}_i = W(P_i^{\optimistic}, s_i)$.
The Whittle index threshold policy $\policy^{\optimistic}_i = \policy_{W^{\optimistic}_i, \lambda}$ achieves the same value function derived from the transition $P_i^{\optimistic}$, which maximizes Equation~\eqref{eqn:optimistic-whittle}.
Aggregating all the arms together, optimistic policy~$\policy^{\optimistic}$ with optimistic transitions~$\transitions^{\optimistic}$ maximizes the future value of the current state~$\statevec$.

\subsection{UCWhittle Algorithm}
After computing optimistic transitions and the corresponding optimistic Whittle indices \eqref{eqn:optimistic-whittle-heuristic}, we execute the optimistic Whittle index threshold policy. The full algorithm is outlined in Algorithm~\ref{alg:ucwhittle}, and implementation details --- including novel techniques for speeding up the computation of the Whittle index --- are given in Appendix~\ref{sec:implementation-details}.

\begin{algorithm}
\caption{UCWhittle}
\label{alg:ucwhittle}
\begin{algorithmic}[1]
\STATE \textbf{Input:} $N$~arms, budget~$K$, episode horizon~$H$ 
\STATE Initialize counts $N_i^{(\round)}(s, a, s') = 0$ for all $s, a, s'$
\STATE Randomly initialize penalty~$\lambda^{(1)}$
\FOR{episode $\round \in \{1, 2, \ldots \}$}
\STATE Reset $h = 1$ and $\statevec = \statevec_\text{init}$ \COMMENT{Reset RMAB instance}
\STATE $P^{\optimistic}_i = \text{\ref{eqn:optimistic-whittle}}(s_i, N_i^{(\round)}, \lambda^{(t)})$ for all $i \in [N]$ \COMMENT{Compute an optimistic transition for each arm} 
\STATE $W_i = \textsc{ComputeWI}(P_i^{\optimistic}, s_i)$  for all $i \in [N]$ \COMMENT{Compute Whittle indices using Def.~\ref{def:whittle-index}}
\STATE Execute $\policy^{(\round)}$ for $\Horizon$ steps by pulling arms with the top $K$ Whittle indices. Observe transitions $(\statevec, \actionvec, \statevec')$
\STATE Update counts $N^{(\round)}_i$, empirical means $\hat{\transitions}^{(\round)}$, and confidence regions $\boldsymbol B^{(\round)}$
\STATE $\lambda^{(\round+1)} = K$th highest Whittle index \COMMENT{Update penalty}
\ENDFOR
\end{algorithmic}
\end{algorithm}

\subsection{Alternative Formulation for Whittle Index Upper Bound}
Equation~\eqref{eqn:optimistic-whittle} provides optimistic transition probabilities but requires separately solving for optimistic Whittle indices afterwards. Computing a Whittle index involves binary search, solving value iteration at every step, so is quite computationally expensive. We thus formulate a heuristic which solves for the highest {\it Whittle index} directly (instead of highest {\it future value}) at the current state~$s_{\horizon,i}$:
\begin{align}\label{eqn:optimistic-whittle-heuristic}
    & \max\limits_{m_i, V, Q, P_i, \in \boldsymbol B_i^{(\round)}} \quad  m_i ~\tag{$\maxPenalty$} \\
    & \text{s.t.} ~~ V(s) = \max\limits_{a \in \Action} Q(s,a), ~~ Q(s,a\!=\!0) = Q(s,a\!=\!1) \nonumber \\
    & Q(s,a) = -m_i a + R(s,a) + \gamma \sum\nolimits_{s' \in \stateset} P_i(s,a,s') V(s') \nonumber
\end{align}

Solving Equation~\eqref{eqn:optimistic-whittle-heuristic} directly yields the maximal Whittle index estimate within the confidence bound. We thus save computation cost while maintaining a valid upper bound to the optimistic Whittle index from Equation~\eqref{eqn:optimistic-whittle}.
The theoretical analysis does not hold for \eqref{eqn:optimistic-whittle-heuristic}, but empirically, we show that this heuristic achieves comparable performance with significantly lower computation.

\input{theory}

\begin{figure*}[th]
    \centering
\input{figures/plot_budget_ratio}

    \caption{Varying budget ratio $K / N$, with $N = 15$ arms, on the ARMMAN domain. Our UCWhittle approaches perform stronger than baselines, particularly in the challenging low-budget scenarios.}
    \label{fig:experiment-budget-ratio}
\end{figure*}

\begin{figure*}[th]
    \centering
\input{figures/plot_episode_len}

    \caption{Changing episode length $H$ on the ARMMAN domain. We run each setting for 1,200 total timesteps. \emph{UCW-penalty} performs best with longer horizons. At shorter horizons, \emph{UCW-value} converges in fewer timesteps, but more episodes are necessary: around episode $t = 100$ with a horizon $H=5$ compared to episode $t = 16$ with horizon $H = 50$.
    }
    \label{fig:experiment-episode-length}
\end{figure*}

\section{Experiments}
We show that UCWhittle achieves consistently low regret across three domains, including one generated from real-world data on maternal health. Additional details about the dataset and data usage are in Appendix~\ref{sec:consent}, and details about implementation (including novel techniques to speed up computation) and experiments are in Appendix~\ref{sec:experiment-addl-details}.\footnote{Code available at \url{https://github.com/lily-x/online-rmab}}

\subsection{Preliminaries}

\paragraph{Domains} We consider three binary-action, binary-state settings. Across all domains, the binary states are \emph{good} or \emph{bad}, with reward $1$ and $0$ respectively. 
We impose two assumptions: that acting is always beneficial (more likely to transition to the good state), and that it is always better to start from the good state (more likely to stay in good state).

\textit{ARMMAN} is a non-profit based in India that disseminates health information to pregnant women and mothers to reduce maternal mortality. Twice a week, ARMMAN sends automated voice messages to enrolled mothers relaying critical preventative health information. To improve listenership, the organization provides service calls to a subset of mothers; the challenge is selecting which subset to call to maximize engagement. We use real, anonymized data of the engagement behavior of 7,656 mothers from a previous RMAB field study~\cite{mate2022field}.
We construct instances of RMAB problem with transition probabilities randomly sampled from the real dataset.



\textit{Wide Margin} \quad We randomly generate transition probabilities with high variance, while respecting the constraints specified above. 

\textit{Thin Margin} \quad For a more challenging setting, we consider a synthetic domain with probabilities of transitioning to the good state constrained to the interval $[0.2, 0.4]$ to test the ability of each approach to discern smaller differences in transition probabilities.



\paragraph{Algorithms} 
We evaluate both variants of UCWhittle (Algorithm~\ref{alg:ucwhittle}) introduced in this paper.
\quad \emph{UCWhittle-value} uses the value-maximizing bilinear program \eqref{eqn:optimistic-whittle} while \emph{UCWhittle-penalty} uses the penalty-maximizing bilinear program \eqref{eqn:optimistic-whittle-heuristic}. 

In this paper, we focus on frequentist regret, thus we exclude the Bayesian regret baselines, e.g., Thompson sampling~\cite{jung2019regret}, because their regret bounds are averaged over a prior.
We consider the following three regret baselines:
\ \emph{ExtremeWhittle} is similar to the the approach by \citet{wang2019opportunistic}: estimate Whittle indices from the extreme points of the unknown transition probabilities, using UCBs of active transition probabilities and lower confidence bounds (LCB) for passive transition probabilities to estimate the gap between the value of acting versus not acting. We then solve a Whittle index policy using these estimates.
\ \emph{WIQL} \citep{biswas2021learn} uses Q-learning to learn the value function of each arm at each state by interacting with the RMAB instance. 
\ \emph{Random} takes a random action at each step, serving as a baseline for expected reward without using any strategic learning algorithm.
Lastly, we evaluate an \emph{optimal} policy which computes a Whittle index policy with access to the true transition probabilities.




\paragraph{Experiment setup}

We evaluate the performance of each algorithm across $T$ episodes of length $H$. The per-episode reward is the cumulative discounted reward with discount rate $\gamma = 0.9$. We then compute regret by subtracting the reward earned by each algorithm from the reward of the \emph{optimal} policy. 
Results are averaged over 30 random seeds and smoothed using exponential smoothing with a weight of $0.9$. We ensure consistency by enforcing, across all algorithms, identical populations (transition probabilities for each arm) and initial state for each episode. 


\subsection{Results}
The performance results across all three domains are shown in Figure~\ref{fig:experiment}. Our UCWhittle algorithm using the value-maximizing bilinear program (\emph{UCW-value}) achieves consistently strong performance and generally converges by 600 timesteps (across varying episode lengths). 
In Figures \ref{fig:experiment-budget-ratio} and~\ref{fig:experiment-episode-length} we evaluate performance while varying the budget~$K$ and episode length~$H$, as the regret of UCWhittle (Theorem~\ref{thm:regret-bound}) has dependency on both the budget as a ratio of total number of arms ($K/N$) and episode length $H$. We see that UCW-value performs comparatively stronger than the baselines in the challenging low-budget settings, in which each arm pull has greater impact. 


Our heuristic approach \emph{UCW-penalty} --- the penalty-maximizing bilinear program we present in Equation~\eqref{eqn:optimistic-whittle-heuristic} --- shows strong performance. UCW-penalty performs even better than UCW-value in some settings, particularly in the ARMMAN domain with $N=15$ arms (Figure~\ref{fig:experiment-budget-ratio}). Notably in Table~\ref{table:runtime} we see this heuristic approach performs dramatically faster than UCW-value --- a $6.1\times$ speedup. Therefore while are able to establish regret guarantees only for UCW-value, we also propose UCW-penalty as a strong candidate for its strong performance and quick execution.

\begin{table}[t]
\centering
\begin{tabular}{lr}
\toprule
\textbf{Method} & \textbf{Time (s)} \\
\midrule
UCWhittle-value & 1090.92 \\
UCWhittle-penalty & 177.57 \\
ExtremeWhittle & 109.44 \\
WIQL & 3.39 \\
random & 1.32 \\
\bottomrule
\end{tabular}
\caption{Average runtime of the different approaches across 500 timesteps with $N=30$ arms and budget $B=6$} 
\label{table:runtime}
\end{table}

In Figures \ref{fig:experiment-budget-ratio} and~\ref{fig:experiment-episode-length} we see \emph{ExtremeWhittle} has poor performance particularly in the early episodes, consistently achieving higher regret than the random policy. Additionally, \emph{WIQL} is slow to converge, performing similarly to the random baseline across the time horizons that we consider.

\section{Conclusion}
\label{sec:conclusion}
We propose the first online learning algorithm for RMABs based on the Whittle index policy, using an upper confidence bound--approach to learn transition dynamics. We formulate a bilinear program to compute optimistic Whittle indices from the confidence bounds of transition dynamics, enabling online learning using an optimistic Whittle index threshold policy. Theoretically, our work pushes the boundary of existing frequentist regret bounds in RMABs while enabling scalability using the Whittle index threshold policy to decompose the solution approach. 

\section*{Acknowledgments}
Kai Wang was supported by W911NF-17-1-0370 and ARO Grant Number W911NF-18-1-0208. Lily Xu was supported by ARO Grant Number W911NF-18-1-0208. The views and conclusions contained in this document are those of the authors and should not be interpreted as representing the official policies, either expressed or implied, of ARO or the U.S.\ Government. The U.S.\ Government is authorized to reproduce and distribute reprints for Government purposes notwithstanding any copyright notation herein.
Kai Wang was also supported by Siebel Scholars.
Lily Xu was also supported by a Google PhD fellowship. 
The authors thank Elias Khalil and anonymous reviewers for their thoughtful comments.



%% file: theory.tex
\section{Regret Analysis}
\label{sec:regret}


We analyze the regret of our UCWhittle algorithm to provide the first frequentist regret analysis for RMABs.
In this section, we use the Lagrangian objective as a proxy to the reward received from the proposed policy.
Section~\ref{sec:regret-known-penalty} first assumes an arbitrary penalty~$\lambda$ is given to define the regret (Definition~\ref{def:lagrangian-regret}). Section~\ref{sec:regret-unknown-penalty} generalizes by defining the regret of the optimal Lagrangian objective based on the unknown optimal penalty~$\lambda^\star$ (Definition~\ref{def:optimal-lagrangian-regret}). Section~\ref{sec:penalty-update} provides an update rule for updating the penalty~$\lambda^{(\round)}$ after each episode. Full proofs are given in Appendix~\ref{sec:full-proofs}.


\begin{figure*}[th]
    \centering

    \input{figures/plot_n40}
    
    \caption{
   Cumulative discounted regret (lower is better) in each episode ($x$-axis) incurred by our UCWhittle approaches compared to baselines across the three domains with 
    $N=8$ arms, budget $B=3$,
    episode length $H=20$, and $T=40$ episodes.
}
    \label{fig:experiment}
\end{figure*}

\subsection{Regret Bound with Known Penalty}
\label{sec:regret-known-penalty}
By the Chernoff bound, we know that with high probability the true transition $\truep$ lies within $\boldsymbol B^{(\round)}$:

\begin{restatable}[]{proposition}{chernoffBound}
Given $\delta > 0$ and $t \geq 1$, we have: $\Pr \left(\truep \in \boldsymbol B^{(\round)} \right) \geq 1 - \frac{\delta}{\round^4}$.
\end{restatable}
This bound can be used to bound the regret incurred, even when the confidence bound fails. In the following theorem, we bound the regret in the case where the confidence bound holds and when the penalty $\lambda$ is given.

\begin{theorem}[Regret decomposition]\label{thm:regret-decomposition-all}
Given the penalty $\lambda$ and $\truep \in \boldsymbol B^{(\round)}$ for all $\round$, we have:
\begin{align}\label{eqn:regret-decomposition-all}
    \text{Reg}_\lambda(\Round) & = \sum\nolimits_{\round \in [T]} U^{\truep, \lambda}_{\optpolicy}(\boldsymbol s_1) - U^{\truep, \lambda}_{\policy^{(\round)}}(\boldsymbol s_1) \nonumber \\
    & \leq \sum\nolimits_{\round \in [T]} U^{\transitions^{(\round)}, \lambda}_{\policy^{(\round)}}(\boldsymbol s_1) - U^{\truep,\lambda}_{\policy^{(\round)}}(\boldsymbol s_1) \ .
\end{align}
\end{theorem}
\begin{proof}
By optimality of Equation~\eqref{eqn:optimistic-whittle} to enable $(P^{(\round)}_i, \policy^{(\round)}_i) = \arg\max\nolimits_{P_i \in B^{(\round)}_i, \policy_i} V^{P_i, \lambda}_{\policy_i}(s_{1,i})$ and the assumption that the true transition lies within the confidence region $P^\star_i \in B_i^{(\round)}$, we show that:
\begin{align*}
    U^{\truep, \lambda}_{\optpolicy}(\boldsymbol s_1) & = \sum\nolimits_{i \in [N]} V_{\optpolicy_i}^{P^\star_i, \lambda}(s_{1,i}) \nonumber \\
    & \leq \sum\nolimits_{i \in [N]} V_{\policy^{(\round)}_i}^{P^{(\round)}_i, \lambda}(s_{1,i}) = U^{\transitions^{(\round)}, \lambda}_{\policy^{(\round)}}(\boldsymbol s_1) \ .\qedhere
\end{align*}
\end{proof}

Theorem~\ref{thm:regret-decomposition-all} enables us to bound our regret by the difference between two future values under the same policy $\policy^{(\round)}$. 

\begin{definition}[Bellman operator]
Define the \emph{Bellman operator} as: 
\begin{align}
    \mathcal{T}^{P_i}_{\policy_i} V(s) \! = \! \mathop{\mathbb{E}}\limits_{a \sim \policy_i} \left[ -\lambda a \! + \! R(s,a) \! + \! \gamma \sum\limits_{s' \in S} P_i(s,a,s') V(s') \right] \nonumber
\end{align}
\end{definition}

Using Theorem~\ref{thm:regret-decomposition-all} and the Bellman operator, we can further decompose the regret as:

\begin{restatable}[Per-episode regret decomposition in the fully observable setting]{theorem}{perEpisode}\label{thm:regret-decomposition}
    For an arm~$i$, fix $P^{(\round)}_i$, $P^\star_i$, $\lambda$, and the initial state $s_{1,i}$. We have:
    \begin{align}\label{eqn:regret-decomposition}
        & V^{P^{(\round)}_i, \lambda}_{\policy^{(\round)}_i}(s_{1,i}) - V^{P^\star_i, \lambda}_{\policy^{(\round)}_i}(s_{1,i}) = \nonumber \\
        & \mathop{\mathbb{E}}_{P^\star_i, \policy^{(\round)}_i} \left[ \sum\limits_{\horizon=1}^{\infty} \gamma^{\horizon-1} \left(\mathcal{T}^{P^{(\round)}_i}_{\policy^{(\round)}_i} - \mathcal{T}^{P^\star_i}_{\policy^{(\round)}_i} \right) V^{P^{(\round)}_i, \lambda}_{\policy^{(\round)}_i}(s_{\horizon,i}) \right] \ .
    \end{align}
\end{restatable}
Theorem~\ref{thm:regret-decomposition} further decomposes the regret in Equation~\ref{eqn:regret-decomposition-all} into individual differences in Bellman operators. The next theorem bounds the differences in Bellman operators by differences in transition probabilities.


\begin{restatable}[]{theorem}{regretBound}\label{thm:regret-bound}
Assume the penalty term $\lambda^{(\round)} = \lambda$ is given and the RMAB instance is $\epsilon$-ergodicity after $H$~timesteps. Then with probability $1 - \delta$, the cumulative regret in $\Round$ episodes is:
\begin{align}
    \text{Reg}_\lambda(\round) \leq O\left( \frac{1}{\epsilon} |S| |A|^{\frac{1}{2}} N H \sqrt{\Round \log \Round} \right) \ .
\end{align}
\end{restatable}
\begin{proof}[Proof sketch]
We focus on bounding the regret when the confidence bounds hold. 
By Theorem~\ref{thm:regret-decomposition-all} and Theorem~\ref{thm:regret-decomposition}, we estimate the right-hand side of Equation~\ref{eqn:regret-decomposition} to bound the total regret by the $L^1$-difference in the transition probability:
\begin{align}
    & \sum\nolimits_{\horizon=1}^{\infty} \gamma^{\horizon-1} \left( \mathcal{T}^{P^{(\round)}_i}_{\policy^{(\round)}_i} - \mathcal{T}^{P^\star_i}_{\policy^{(\round)}_i} \right) V^{P^{(\round)}_i}_{\policy^{(\round)}_i}(s_{\horizon,i}) \label{eqn:regret-bound-by-transition} \\ 
    \leq & \sum\nolimits_{h=1}^{\infty} \gamma^{h-1} \norm{P^{(\round)}_i(s_{h,i}, a_{h,i}, \cdot) - P^\star_i(s_{h,i}, a_{h,i}, \cdot)}_1 V_{\text{max}}. \nonumber
\end{align}
We bound the regret outside of the horizon~$\Horizon$ by the ergodic assumption of the MDPs.
For the regret inside the horizon~$\Horizon$, we use the confidence radius to bound the $L^1$-norm of transition probability differences and count the number of observations for each state--action pair to express the regret as a sequence of random variables, whose sum can be bounded by Lemma~\ref{lemma:sum-of-sqrt} to conclude the proof.
\end{proof}

When the penalty term $\lambda$ is given, Theorem~\ref{thm:regret-bound} bounds the frequentist regret with a constant term depending on the ergodicity $\epsilon$ of the underlying true MDPs.

\subsection{Regret Bound with Unknown Optimal Penalty}
\label{sec:regret-unknown-penalty}
The analysis in Theorem~\ref{thm:regret-decomposition-all} assumes a fixed and given penalty $\lambda$. Now, we generalize to regret defined in terms of the optimal but unknown penalty~$\lambda^\star$ (Definition~\ref{def:optimal-lagrangian-regret}). We show that updating penalty $\lambda^{(\round)}$ in Algorithm~\ref{alg:ucwhittle} achieves the same regret bound without requiring knowledge of the true transitions~$\truep$ or optimal penalty~$\lambda^\star$:

\begin{restatable}[Regret bound with optimal penalty]{theorem}{agnosticRegretBound}\label{thm:regret-bound-optimal}
Assume the penalty $\lambda^{(\round)}$ in Algorithm~\ref{alg:ucwhittle} is updated by a saddle point $(\lambda^{(\round)}, \transitions^{(\round)}, \policy^{(\round)}) = \arg\min_{\lambda}\max_{\transitions, \policy} U^{\transitions, \lambda}_{\policy}(\boldsymbol s_1)$ subject to constraints in Equation~\eqref{eqn:optimistic-whittle}.
The cumulative regret of the optimal Lagrangian objective is bounded with probability $1 - \delta$:
\begin{align}
    \text{Reg}_{\lambda^\star}(\round) \leq O\left(\frac{1}{\epsilon} |S| |A|^{\frac{1}{2}} N H \sqrt{\Round \log \Round} \right) \ .
\end{align}
\end{restatable}
\begin{proof}[Proof sketch]
The main challenge of an unknown penalty term $\lambda^\star$ is that the optimality of the chosen transition $\transitions^{(\round)}$ and policy $\policy^{(\round)}$ does not hold in Theorem~\ref{thm:regret-decomposition-all} due to the misalignment of the penalty $\lambda^{(\round)}$ used in solving Equation~\eqref{eqn:optimistic-whittle} and the penalty $\lambda^\star$ used in the regret.

Surprisingly, the optimality of $(\lambda^{(\round)}, \transitions^{(\round)}, \policy^{(\round)}) = \arg\min_{\lambda}\max_{\transitions, \policy} U^{\transitions, \lambda}_{\policy}(\boldsymbol s_1)$ and $\lambda^\star = \inf\nolimits_{\lambda} U^{\truep, \lambda}_{\optpolicy}(\boldsymbol s_1)$ is sufficient to show Theorem~\ref{thm:regret-decomposition-all} by:
\begin{align}
    & \underbrace{U^{\truep, \lambda^\star}_{\optpolicy} \quad \leq}_{\lambda^\star \text{ minimizes } U^{\truep, \lambda}_{\optpolicy}} \underbrace{U^{\truep, \lambda^{(\round)}}_{\optpolicy} \leq U^{\boldsymbol P^{(\round)}, \lambda^{(\round)}}_{\policy^{(\round)}}}_{\text{$\transitions^{(t)}, \policy^{(t)}$ maximizes $U^{\transitions, \lambda^{(t)}}_{\policy}$}} \underbrace{\leq \quad U^{\boldsymbol P^{(\round)}, \lambda^\star}_{\policy^{(\round)}}}_{\lambda^{(t)} \text{ minimizes } U^{\transitions^{(t)}, \lambda}_{\policy^{(t)}}} \nonumber \\ 
    & \Longrightarrow ~\ \text{Reg}_{\lambda^\star}^{(\round)} = U^{\truep, \lambda^\star}_{\optpolicy} - U^{\truep, \lambda^\star}_{\policy^{(\round)}} \leq U^{\boldsymbol P^{(\round)}, \lambda^\star}_{\policy^{(\round)}} - U^{\truep, \lambda^\star}_{\policy^{(\round)}} \ . \label{eqn:regret-analysis-unknown-penalty2}
\end{align}
where we omit the dependency on $\boldsymbol s_1$.

After taking summation over $t \in [T]$, Equation~\ref{eqn:regret-analysis-unknown-penalty2} leads to the same result as Theorem~\ref{thm:regret-decomposition-all} without requiring knowledge of the optimal penalty $\lambda^\star$.
The rest of the proof follows the same argument in Theorem~\ref{thm:regret-decomposition} and Theorem~\ref{thm:regret-bound} with the same regret bound.
\end{proof}


\subsection{Penalty Update Rule}
\label{sec:penalty-update}
Theorem~\ref{thm:regret-bound-optimal} suggests that the penalty term $\lambda^{(\round)}$ should be defined by solving a minimax problem $(\lambda^{(\round)}, \transitions^{(\round)}, \policy^{(\round)}) = \arg\min_{\lambda}\max_{\transitions, \policy} U^{\transitions, \lambda}_{\policy}(\boldsymbol s_1)$. However, the bilinear objective of $\maxValue$ --- where the transition probability and value function variables are being multiplied together --- is difficult to solve in a  minimax problem. A heuristic solution is to solve the maximization problem using the previous penalty $\lambda^{(t-1)}$ to determine $\boldsymbol P^{(\round)}$ and $\policy^{(\round)}$ (Equation~\eqref{eqn:optimistic-whittle}). We update $\lambda^{(\round)}$ based on the current policy, set equal to the $K$th largest Whittle index pulled at time~$\round$ to minimize the Lagrangian. This update rule mimics the minimax update rule required by Theorem~\ref{thm:regret-bound-optimal}.

%% file: figures/plot_n40.tex
\begin{tikzpicture}
\pgfplotsset{
  width=0.35\linewidth,
  height=0.2\linewidth,
  ytick pos=left,
  tick label style={font=\scriptsize},
  ymajorgrids=true,
  xtick style={draw=none},
}

\begin{axis}[
  at={(0.32\linewidth, 0\linewidth)}, 
  table/col sep=comma,
xlabel style={align=center},
  xlabel={\small{(b)~Wide Margin}},
  ylabel style={align=center},
  legend style={
    at={(2.1,1.25)},
    legend columns=5,
    font=\small,draw=none,fill=none},
  xmin=0, xmax=40,
  error bars/y dir=both, 
  error bars/y explicit,  
  error bars/error bar style={color=black, thick},
]

\addplot[blue, line width=1pt, mark=none] table [x=episode, y=ucw_value_avg] {data/per_ep_regret_synthetic_n8_b3_s2_a2_H20_L60_epochs30.csv}; \addlegendentry{~UCW-value~\eqref{eqn:optimistic-whittle}~~~~~};
\addplot[cyan, line width=1pt, mark=none] table [x=episode, y=ucw_qp_avg] {data/per_ep_regret_synthetic_n8_b3_s2_a2_H20_L60_epochs30.csv}; \addlegendentry{~UCW-penalty~\eqref{eqn:optimistic-whittle-heuristic}~~~~~};
\addplot[ForestGreen, densely dashdotted, line width=1pt, mark=none] table [x=episode, y=ucw_extreme_avg] {data/per_ep_regret_synthetic_n8_b3_s2_a2_H20_L60_epochs30.csv}; \addlegendentry{~ExtremeWhittle~~~~~};
\addplot[BurntOrange, densely dashed, line width=1pt, mark=none] table [x=episode, y=wiql_avg] {data/per_ep_regret_synthetic_n8_b3_s2_a2_H20_L60_epochs30.csv}; \addlegendentry{~WIQL~~~~~};
\addplot[gray, densely dotted, line width=.7pt, mark=none] table [x=episode, y=random_avg] {data/per_ep_regret_synthetic_n8_b3_s2_a2_H20_L60_epochs30.csv}; \addlegendentry{~random};


\addplot [draw=none, mark=none, name path=random_low] table [x=episode, y expr=\thisrowno{11}+\thisrowno{12}, col sep=comma] {data/per_ep_regret_synthetic_n8_b3_s2_a2_H20_L60_epochs30.csv};
\addplot [draw=none, mark=none, name path=random_high] table [x=episode, y expr=\thisrowno{11}-\thisrowno{12}, col sep=comma] {data/per_ep_regret_synthetic_n8_b3_s2_a2_H20_L60_epochs30.csv};
\tikzfillbetween[of=random_low and random_high, on layer=main]{gray, opacity=0.1};

\addplot [draw=none, mark=none, name path=wiql_low] table [x=episode, y expr=\thisrowno{9}+\thisrowno{10}, col sep=comma] {data/per_ep_regret_synthetic_n8_b3_s2_a2_H20_L60_epochs30.csv};
\addplot [draw=none, mark=none, name path=wiql_high] table [x=episode, y expr=\thisrowno{9}-\thisrowno{10}, col sep=comma] {data/per_ep_regret_synthetic_n8_b3_s2_a2_H20_L60_epochs30.csv};
\tikzfillbetween[of=wiql_low and wiql_high, on layer=main]{BurntOrange, opacity=0.1};

\addplot [draw=none, mark=none, name path=ucw_extreme_low] table [x=episode, y expr=\thisrowno{7}+\thisrowno{8}, col sep=comma] {data/per_ep_regret_synthetic_n8_b3_s2_a2_H20_L60_epochs30.csv};
\addplot [draw=none, mark=none, name path=ucw_extreme_high] table [x=episode, y expr=\thisrowno{7}-\thisrowno{8}, col sep=comma] {data/per_ep_regret_synthetic_n8_b3_s2_a2_H20_L60_epochs30.csv};
\tikzfillbetween[of=ucw_extreme_low and ucw_extreme_high, on layer=main]{ForestGreen, opacity=0.1};

\addplot [draw=none, mark=none, name path=ucw_subsidy_low] table [x=episode, y expr=\thisrowno{5}+\thisrowno{6}, col sep=comma] {data/per_ep_regret_synthetic_n8_b3_s2_a2_H20_L60_epochs30.csv};
\addplot [draw=none, mark=none, name path=ucw_subsidy_high] table [x=episode, y expr=\thisrowno{5}-\thisrowno{6}, col sep=comma] {data/per_ep_regret_synthetic_n8_b3_s2_a2_H20_L60_epochs30.csv};
\tikzfillbetween[of=ucw_subsidy_low and ucw_subsidy_high, on layer=main]{cyan, opacity=0.1};

\addplot [draw=none, mark=none, name path=ucw_value_low] table [x=episode, y expr=\thisrowno{3}+\thisrowno{4}, col sep=comma] {data/per_ep_regret_synthetic_n8_b3_s2_a2_H20_L60_epochs30.csv};
\addplot [draw=none, mark=none, name path=ucw_value_high] table [x=episode, y expr=\thisrowno{3}-\thisrowno{4}, col sep=comma] {data/per_ep_regret_synthetic_n8_b3_s2_a2_H20_L60_epochs30.csv};
\tikzfillbetween[of=ucw_value_low and ucw_value_high, on layer=main]{blue, opacity=0.05};
\end{axis}

\begin{axis}[
  at={(.0\linewidth, .0\linewidth)}, 
  table/col sep=comma,
  xlabel={\footnotesize{(a)~ARMMAN}},
  ylabel={\footnotesize{Regret}},
   xmin=0, xmax=40,
  error bars/y dir=both, 
  error bars/y explicit,  
  error bars/error bar style={color=black, thick},
]

\addplot[blue, line width=1pt, mark=none] table [x=episode, y=ucw_value_avg] {data/per_ep_regret_kai_n8_b3_s2_a2_H20_L60_epochs30.csv}; 

\addplot[cyan, line width=1pt, mark=none] table [x=episode, y=ucw_qp_avg] {data/per_ep_regret_kai_n8_b3_s2_a2_H20_L60_epochs30.csv}; 
\addplot[ForestGreen, densely dashdotted, line width=1pt, mark=none] table [x=episode, y=ucw_extreme_avg] {data/per_ep_regret_kai_n8_b3_s2_a2_H20_L60_epochs30.csv}; 
\addplot[BurntOrange, densely dashed, line width=1pt, mark=none] table [x=episode, y=wiql_avg] {data/per_ep_regret_kai_n8_b3_s2_a2_H20_L60_epochs30.csv}; 
\addplot[gray, densely dotted, line width=.7pt, mark=none] table [x=episode, y=random_avg] {data/per_ep_regret_kai_n8_b3_s2_a2_H20_L60_epochs30.csv}; 


\addplot [draw=none, mark=none, name path=random_low] table [x=episode, y expr=\thisrowno{11}+\thisrowno{12}, col sep=comma] {data/per_ep_regret_kai_n8_b3_s2_a2_H20_L60_epochs30.csv};
\addplot [draw=none, mark=none, name path=random_high] table [x=episode, y expr=\thisrowno{11}-\thisrowno{12}, col sep=comma] {data/per_ep_regret_kai_n8_b3_s2_a2_H20_L60_epochs30.csv};
\tikzfillbetween[of=random_low and random_high, on layer=main]{gray, opacity=0.1};

\addplot [draw=none, mark=none, name path=wiql_low] table [x=episode, y expr=\thisrowno{9}+\thisrowno{10}, col sep=comma] {data/per_ep_regret_kai_n8_b3_s2_a2_H20_L60_epochs30.csv};
\addplot [draw=none, mark=none, name path=wiql_high] table [x=episode, y expr=\thisrowno{9}-\thisrowno{10}, col sep=comma] {data/per_ep_regret_kai_n8_b3_s2_a2_H20_L60_epochs30.csv};
\tikzfillbetween[of=wiql_low and wiql_high, on layer=main]{BurntOrange, opacity=0.1};

\addplot [draw=none, mark=none, name path=ucw_extreme_low] table [x=episode, y expr=\thisrowno{7}+\thisrowno{8}, col sep=comma] {data/per_ep_regret_kai_n8_b3_s2_a2_H20_L60_epochs30.csv};
\addplot [draw=none, mark=none, name path=ucw_extreme_high] table [x=episode, y expr=\thisrowno{7}-\thisrowno{8}, col sep=comma] {data/per_ep_regret_kai_n8_b3_s2_a2_H20_L60_epochs30.csv};
\tikzfillbetween[of=ucw_extreme_low and ucw_extreme_high, on layer=main]{ForestGreen, opacity=0.1};

\addplot [draw=none, mark=none, name path=ucw_subsidy_low] table [x=episode, y expr=\thisrowno{5}+\thisrowno{6}, col sep=comma] {data/per_ep_regret_kai_n8_b3_s2_a2_H20_L60_epochs30.csv};
\addplot [draw=none, mark=none, name path=ucw_subsidy_high] table [x=episode, y expr=\thisrowno{5}-\thisrowno{6}, col sep=comma] {data/per_ep_regret_kai_n8_b3_s2_a2_H20_L60_epochs30.csv};
\tikzfillbetween[of=ucw_subsidy_low and ucw_subsidy_high, on layer=main]{cyan, opacity=0.1};

\addplot [draw=none, mark=none, name path=ucw_value_low] table [x=episode, y expr=\thisrowno{3}+\thisrowno{4}, col sep=comma] {data/per_ep_regret_kai_n8_b3_s2_a2_H20_L60_epochs30.csv};
\addplot [draw=none, mark=none, name path=ucw_value_high] table [x=episode, y expr=\thisrowno{3}-\thisrowno{4}, col sep=comma] {data/per_ep_regret_kai_n8_b3_s2_a2_H20_L60_epochs30.csv};
\tikzfillbetween[of=ucw_value_low and ucw_value_high, on layer=main]{blue, opacity=0.05};
\end{axis}

\begin{axis}[
  at={(.64\linewidth, .0\linewidth)}, 
  table/col sep=comma,
  xlabel={\footnotesize{(c)~Thin Margin}},
  xmin=0, xmax=40,
  error bars/y dir=both, 
  error bars/y explicit,  
  error bars/error bar style={color=black, thick},
]

\addplot[blue, line width=1pt, mark=none] table [x=episode, y=ucw_value_avg] {data/per_ep_regret_perturb-kai_n8_b3_s2_a2_H20_L60_epochs3.csv}; 
\addplot[cyan, line width=1pt, mark=none] table [x=episode, y=ucw_qp_avg] {data/per_ep_regret_perturb-kai_n8_b3_s2_a2_H20_L60_epochs3.csv}; 
\addplot[ForestGreen, densely dashdotted, line width=1pt, mark=none] table [x=episode, y=ucw_extreme_avg] {data/per_ep_regret_perturb-kai_n8_b3_s2_a2_H20_L60_epochs3.csv}; 
\addplot[BurntOrange, densely dashed, line width=1pt, mark=none] table [x=episode, y=wiql_avg] {data/per_ep_regret_perturb-kai_n8_b3_s2_a2_H20_L60_epochs3.csv}; 
\addplot[gray, densely dotted, line width=.7pt, mark=none] table [x=episode, y=random_avg] {data/per_ep_regret_perturb-kai_n8_b3_s2_a2_H20_L60_epochs3.csv}; 


\addplot [draw=none, mark=none, name path=random_low] table [x=episode, y expr=\thisrowno{11}+\thisrowno{12}, col sep=comma] {data/per_ep_regret_perturb-kai_n8_b3_s2_a2_H20_L60_epochs3.csv};
\addplot [draw=none, mark=none, name path=random_high] table [x=episode, y expr=\thisrowno{11}-\thisrowno{12}, col sep=comma] {data/per_ep_regret_perturb-kai_n8_b3_s2_a2_H20_L60_epochs3.csv};
\tikzfillbetween[of=random_low and random_high, on layer=main]{gray, opacity=0.1};

\addplot [draw=none, mark=none, name path=wiql_low] table [x=episode, y expr=\thisrowno{9}+\thisrowno{10}, col sep=comma] {data/per_ep_regret_perturb-kai_n8_b3_s2_a2_H20_L60_epochs3.csv};
\addplot [draw=none, mark=none, name path=wiql_high] table [x=episode, y expr=\thisrowno{9}-\thisrowno{10}, col sep=comma] {data/per_ep_regret_perturb-kai_n8_b3_s2_a2_H20_L60_epochs3.csv};
\tikzfillbetween[of=wiql_low and wiql_high, on layer=main]{BurntOrange, opacity=0.1};

\addplot [draw=none, mark=none, name path=ucw_extreme_low] table [x=episode, y expr=\thisrowno{7}+\thisrowno{8}, col sep=comma] {data/per_ep_regret_perturb-kai_n8_b3_s2_a2_H20_L60_epochs3.csv};
\addplot [draw=none, mark=none, name path=ucw_extreme_high] table [x=episode, y expr=\thisrowno{7}-\thisrowno{8}, col sep=comma] {data/per_ep_regret_perturb-kai_n8_b3_s2_a2_H20_L60_epochs3.csv};
\tikzfillbetween[of=ucw_extreme_low and ucw_extreme_high, on layer=main]{ForestGreen, opacity=0.1};

\addplot [draw=none, mark=none, name path=ucw_subsidy_low] table [x=episode, y expr=\thisrowno{5}+\thisrowno{6}, col sep=comma] {data/per_ep_regret_perturb-kai_n8_b3_s2_a2_H20_L60_epochs3.csv};
\addplot [draw=none, mark=none, name path=ucw_subsidy_high] table [x=episode, y expr=\thisrowno{5}-\thisrowno{6}, col sep=comma] {data/per_ep_regret_perturb-kai_n8_b3_s2_a2_H20_L60_epochs3.csv};
\tikzfillbetween[of=ucw_subsidy_low and ucw_subsidy_high, on layer=main]{cyan, opacity=0.1};

\addplot [draw=none, mark=none, name path=ucw_value_low] table [x=episode, y expr=\thisrowno{3}+\thisrowno{4}, col sep=comma] {data/per_ep_regret_perturb-kai_n8_b3_s2_a2_H20_L60_epochs3.csv};
\addplot [draw=none, mark=none, name path=ucw_value_high] table [x=episode, y expr=\thisrowno{3}-\thisrowno{4}, col sep=comma] {data/per_ep_regret_perturb-kai_n8_b3_s2_a2_H20_L60_epochs3.csv};
\tikzfillbetween[of=ucw_value_low and ucw_value_high, on layer=main]{blue, opacity=0.05};
\end{axis}

\end{tikzpicture}

%% file: figures/plot_budget_ratio.tex
\begin{tikzpicture}
\pgfplotsset{
  width=0.29\linewidth,
  height=0.2\linewidth,
  ytick pos=left,
  tick label style={font=\scriptsize},
  ymajorgrids=true,
  xtick style={draw=none},
}

\begin{axis}[
  at={(0.24\linewidth, 0\linewidth)}, 
  table/col sep=comma,
  xlabel={\footnotesize{(b)~$K=5$}},
  ylabel style={align=center},
  legend style={
    at={(3.1,1.25)},
    legend columns=5,
    font=\small,draw=none,fill=none},
  xmin=0, xmax=29,
  ymin=10, 
  error bars/y dir=both, 
  error bars/y explicit,  
  error bars/error bar style={color=black, thick},
]

\addplot[blue, line width=1pt, mark=none] table [x=episode, y=ucw_value_avg] {data/per_ep_regret_kai_n15_b5_s2_a2_H20_L30_epochs1.csv}; \addlegendentry{~UCW-value~\eqref{eqn:optimistic-whittle}~~~~~};
\addplot[cyan, line width=1pt, mark=none] table [x=episode, y=ucw_qp_avg] {data/per_ep_regret_kai_n15_b5_s2_a2_H20_L30_epochs1.csv}; \addlegendentry{~UCW-penalty~\eqref{eqn:optimistic-whittle-heuristic}~~~~~};
\addplot[ForestGreen, densely dashdotted, line width=1pt, mark=none] table [x=episode, y=ucw_extreme_avg] {data/per_ep_regret_kai_n15_b5_s2_a2_H20_L30_epochs1.csv}; \addlegendentry{~ExtremeWhittle~~~~~};
\addplot[BurntOrange, densely dashed, line width=1pt, mark=none] table [x=episode, y=wiql_avg] {data/per_ep_regret_kai_n15_b5_s2_a2_H20_L30_epochs1.csv}; \addlegendentry{~WIQL~~~~~};
\addplot[gray, densely dotted, line width=.7pt, mark=none] table [x=episode, y=random_avg] {data/per_ep_regret_kai_n15_b5_s2_a2_H20_L30_epochs1.csv}; \addlegendentry{~random};


\addplot [draw=none, mark=none, name path=random_low] table [x=episode, y expr=\thisrowno{11}+\thisrowno{12}, col sep=comma] {data/per_ep_regret_kai_n15_b5_s2_a2_H20_L30_epochs1.csv};
\addplot [draw=none, mark=none, name path=random_high] table [x=episode, y expr=\thisrowno{11}-\thisrowno{12}, col sep=comma] {data/per_ep_regret_kai_n15_b5_s2_a2_H20_L30_epochs1.csv};
\tikzfillbetween[of=random_low and random_high, on layer=main]{gray, opacity=0.1};

\addplot [draw=none, mark=none, name path=wiql_low] table [x=episode, y expr=\thisrowno{9}+\thisrowno{10}, col sep=comma] {data/per_ep_regret_kai_n15_b5_s2_a2_H20_L30_epochs1.csv};
\addplot [draw=none, mark=none, name path=wiql_high] table [x=episode, y expr=\thisrowno{9}-\thisrowno{10}, col sep=comma] {data/per_ep_regret_kai_n15_b5_s2_a2_H20_L30_epochs1.csv};
\tikzfillbetween[of=wiql_low and wiql_high, on layer=main]{BurntOrange, opacity=0.1};

\addplot [draw=none, mark=none, name path=ucw_extreme_low] table [x=episode, y expr=\thisrowno{7}+\thisrowno{8}, col sep=comma] {data/per_ep_regret_kai_n15_b5_s2_a2_H20_L30_epochs1.csv};
\addplot [draw=none, mark=none, name path=ucw_extreme_high] table [x=episode, y expr=\thisrowno{7}-\thisrowno{8}, col sep=comma] {data/per_ep_regret_kai_n15_b5_s2_a2_H20_L30_epochs1.csv};
\tikzfillbetween[of=ucw_extreme_low and ucw_extreme_high, on layer=main]{ForestGreen, opacity=0.1};

\addplot [draw=none, mark=none, name path=ucw_subsidy_low] table [x=episode, y expr=\thisrowno{5}+\thisrowno{6}, col sep=comma] {data/per_ep_regret_kai_n15_b5_s2_a2_H20_L30_epochs1.csv};
\addplot [draw=none, mark=none, name path=ucw_subsidy_high] table [x=episode, y expr=\thisrowno{5}-\thisrowno{6}, col sep=comma] {data/per_ep_regret_kai_n15_b5_s2_a2_H20_L30_epochs1.csv};
\tikzfillbetween[of=ucw_subsidy_low and ucw_subsidy_high, on layer=main]{cyan, opacity=0.1};

\addplot [draw=none, mark=none, name path=ucw_value_low] table [x=episode, y expr=\thisrowno{3}+\thisrowno{4}, col sep=comma] {data/per_ep_regret_kai_n15_b5_s2_a2_H20_L30_epochs1.csv};
\addplot [draw=none, mark=none, name path=ucw_value_high] table [x=episode, y expr=\thisrowno{3}-\thisrowno{4}, col sep=comma] {data/per_ep_regret_kai_n15_b5_s2_a2_H20_L30_epochs1.csv};
\tikzfillbetween[of=ucw_value_low and ucw_value_high, on layer=main]{blue, opacity=0.05};
\end{axis}

\begin{axis}[
  at={(0\linewidth, .0\linewidth)}, 
  table/col sep=comma,
  xlabel={\footnotesize{(a)~$K=3$}},
  ylabel={\footnotesize{Regret}},
   xmin=0, xmax=29,
   ymin=10, 
  error bars/y dir=both, 
  error bars/y explicit,  
  error bars/error bar style={color=black, thick},
]

\addplot[blue, line width=1pt, mark=none] table [x=episode, y=ucw_value_avg] {data/per_ep_regret_kai_n15_b3_s2_a2_H20_L30_epochs1.csv}; 
\addplot[cyan, line width=1pt, mark=none] table [x=episode, y=ucw_qp_avg] {data/per_ep_regret_kai_n15_b3_s2_a2_H20_L30_epochs1.csv}; 
\addplot[ForestGreen, densely dashdotted, line width=1pt, mark=none] table [x=episode, y=ucw_extreme_avg] {data/per_ep_regret_kai_n15_b3_s2_a2_H20_L30_epochs1.csv}; 
\addplot[BurntOrange, densely dashed, line width=1pt, mark=none] table [x=episode, y=wiql_avg] {data/per_ep_regret_kai_n15_b3_s2_a2_H20_L30_epochs1.csv}; 
\addplot[gray, densely dotted, line width=.7pt, mark=none] table [x=episode, y=random_avg] {data/per_ep_regret_kai_n15_b3_s2_a2_H20_L30_epochs1.csv}; 


\addplot [draw=none, mark=none, name path=random_low] table [x=episode, y expr=\thisrowno{11}+\thisrowno{12}, col sep=comma] {data/per_ep_regret_kai_n15_b3_s2_a2_H20_L30_epochs1.csv};
\addplot [draw=none, mark=none, name path=random_high] table [x=episode, y expr=\thisrowno{11}-\thisrowno{12}, col sep=comma] {data/per_ep_regret_kai_n15_b3_s2_a2_H20_L30_epochs1.csv};
\tikzfillbetween[of=random_low and random_high, on layer=main]{gray, opacity=0.1};

\addplot [draw=none, mark=none, name path=wiql_low] table [x=episode, y expr=\thisrowno{9}+\thisrowno{10}, col sep=comma] {data/per_ep_regret_kai_n15_b3_s2_a2_H20_L30_epochs1.csv};
\addplot [draw=none, mark=none, name path=wiql_high] table [x=episode, y expr=\thisrowno{9}-\thisrowno{10}, col sep=comma] {data/per_ep_regret_kai_n15_b3_s2_a2_H20_L30_epochs1.csv};
\tikzfillbetween[of=wiql_low and wiql_high, on layer=main]{BurntOrange, opacity=0.1};

\addplot [draw=none, mark=none, name path=ucw_extreme_low] table [x=episode, y expr=\thisrowno{7}+\thisrowno{8}, col sep=comma] {data/per_ep_regret_kai_n15_b3_s2_a2_H20_L30_epochs1.csv};
\addplot [draw=none, mark=none, name path=ucw_extreme_high] table [x=episode, y expr=\thisrowno{7}-\thisrowno{8}, col sep=comma] {data/per_ep_regret_kai_n15_b3_s2_a2_H20_L30_epochs1.csv};
\tikzfillbetween[of=ucw_extreme_low and ucw_extreme_high, on layer=main]{ForestGreen, opacity=0.1};

\addplot [draw=none, mark=none, name path=ucw_subsidy_low] table [x=episode, y expr=\thisrowno{5}+\thisrowno{6}, col sep=comma] {data/per_ep_regret_kai_n15_b3_s2_a2_H20_L30_epochs1.csv};
\addplot [draw=none, mark=none, name path=ucw_subsidy_high] table [x=episode, y expr=\thisrowno{5}-\thisrowno{6}, col sep=comma] {data/per_ep_regret_kai_n15_b3_s2_a2_H20_L30_epochs1.csv};
\tikzfillbetween[of=ucw_subsidy_low and ucw_subsidy_high, on layer=main]{cyan, opacity=0.1};

\addplot [draw=none, mark=none, name path=ucw_value_low] table [x=episode, y expr=\thisrowno{3}+\thisrowno{4}, col sep=comma] {data/per_ep_regret_kai_n15_b3_s2_a2_H20_L30_epochs1.csv};
\addplot [draw=none, mark=none, name path=ucw_value_high] table [x=episode, y expr=\thisrowno{3}-\thisrowno{4}, col sep=comma] {data/per_ep_regret_kai_n15_b3_s2_a2_H20_L30_epochs1.csv};
\tikzfillbetween[of=ucw_value_low and ucw_value_high, on layer=main]{blue, opacity=0.05};
\end{axis}

\begin{axis}[
  at={(0.48\linewidth, .0\linewidth)}, 
  table/col sep=comma,
  xlabel={\footnotesize{(c)~$K=7$}},
  xmin=0, xmax=29,
   ymin=10, 
  error bars/y dir=both, 
  error bars/y explicit,  
  error bars/error bar style={color=black, thick},
]

\addplot[blue, line width=1pt, mark=none] table [x=episode, y=ucw_value_avg] {data/per_ep_regret_kai_n15_b7_s2_a2_H20_L30_epochs1.csv}; 
\addplot[cyan, line width=1pt, mark=none] table [x=episode, y=ucw_qp_avg] {data/per_ep_regret_kai_n15_b7_s2_a2_H20_L30_epochs1.csv}; 
\addplot[ForestGreen, densely dashdotted, line width=1pt, mark=none] table [x=episode, y=ucw_extreme_avg] {data/per_ep_regret_kai_n15_b7_s2_a2_H20_L30_epochs1.csv}; 
\addplot[BurntOrange, densely dashed, line width=1pt, mark=none] table [x=episode, y=wiql_avg] {data/per_ep_regret_kai_n15_b7_s2_a2_H20_L30_epochs1.csv}; 
\addplot[gray, densely dotted, line width=.7pt, mark=none] table [x=episode, y=random_avg] {data/per_ep_regret_kai_n15_b7_s2_a2_H20_L30_epochs1.csv}; 


\addplot [draw=none, mark=none, name path=random_low] table [x=episode, y expr=\thisrowno{11}+\thisrowno{12}, col sep=comma] {data/per_ep_regret_kai_n15_b7_s2_a2_H20_L30_epochs1.csv};
\addplot [draw=none, mark=none, name path=random_high] table [x=episode, y expr=\thisrowno{11}-\thisrowno{12}, col sep=comma] {data/per_ep_regret_kai_n15_b7_s2_a2_H20_L30_epochs1.csv};
\tikzfillbetween[of=random_low and random_high, on layer=main]{gray, opacity=0.1};

\addplot [draw=none, mark=none, name path=wiql_low] table [x=episode, y expr=\thisrowno{9}+\thisrowno{10}, col sep=comma] {data/per_ep_regret_kai_n15_b7_s2_a2_H20_L30_epochs1.csv};
\addplot [draw=none, mark=none, name path=wiql_high] table [x=episode, y expr=\thisrowno{9}-\thisrowno{10}, col sep=comma] {data/per_ep_regret_kai_n15_b7_s2_a2_H20_L30_epochs1.csv};
\tikzfillbetween[of=wiql_low and wiql_high, on layer=main]{BurntOrange, opacity=0.1};

\addplot [draw=none, mark=none, name path=ucw_extreme_low] table [x=episode, y expr=\thisrowno{7}+\thisrowno{8}, col sep=comma] {data/per_ep_regret_kai_n15_b7_s2_a2_H20_L30_epochs1.csv};
\addplot [draw=none, mark=none, name path=ucw_extreme_high] table [x=episode, y expr=\thisrowno{7}-\thisrowno{8}, col sep=comma] {data/per_ep_regret_kai_n15_b7_s2_a2_H20_L30_epochs1.csv};
\tikzfillbetween[of=ucw_extreme_low and ucw_extreme_high, on layer=main]{ForestGreen, opacity=0.1};

\addplot [draw=none, mark=none, name path=ucw_subsidy_low] table [x=episode, y expr=\thisrowno{5}+\thisrowno{6}, col sep=comma] {data/per_ep_regret_kai_n15_b7_s2_a2_H20_L30_epochs1.csv};
\addplot [draw=none, mark=none, name path=ucw_subsidy_high] table [x=episode, y expr=\thisrowno{5}-\thisrowno{6}, col sep=comma] {data/per_ep_regret_kai_n15_b7_s2_a2_H20_L30_epochs1.csv};
\tikzfillbetween[of=ucw_subsidy_low and ucw_subsidy_high, on layer=main]{cyan, opacity=0.1};

\addplot [draw=none, mark=none, name path=ucw_value_low] table [x=episode, y expr=\thisrowno{3}+\thisrowno{4}, col sep=comma] {data/per_ep_regret_kai_n15_b7_s2_a2_H20_L30_epochs1.csv};
\addplot [draw=none, mark=none, name path=ucw_value_high] table [x=episode, y expr=\thisrowno{3}-\thisrowno{4}, col sep=comma] {data/per_ep_regret_kai_n15_b7_s2_a2_H20_L30_epochs1.csv};
\tikzfillbetween[of=ucw_value_low and ucw_value_high, on layer=main]{blue, opacity=0.05};
\end{axis}

\begin{axis}[
  at={(0.72\linewidth, .0\linewidth)}, 
  table/col sep=comma,
  xlabel={\footnotesize{(d)~$K=10$}},
  xmin=0, xmax=29,
  ymin=10, 
  error bars/y dir=both, 
  error bars/y explicit,  
  error bars/error bar style={color=black, thick},
]

\addplot[blue, line width=1pt, mark=none] table [x=episode, y=ucw_value_avg] {data/per_ep_regret_kai_n15_b10_s2_a2_H20_L30_epochs1.csv}; 
\addplot[cyan, line width=1pt, mark=none] table [x=episode, y=ucw_qp_avg] {data/per_ep_regret_kai_n15_b10_s2_a2_H20_L30_epochs1.csv}; 
\addplot[ForestGreen, densely dashdotted, line width=1pt, mark=none] table [x=episode, y=ucw_extreme_avg] {data/per_ep_regret_kai_n15_b10_s2_a2_H20_L30_epochs1.csv}; 
\addplot[BurntOrange, densely dashed, line width=1pt, mark=none] table [x=episode, y=wiql_avg] {data/per_ep_regret_kai_n15_b10_s2_a2_H20_L30_epochs1.csv}; 
\addplot[gray, densely dotted, line width=.7pt, mark=none] table [x=episode, y=random_avg] {data/per_ep_regret_kai_n15_b10_s2_a2_H20_L30_epochs1.csv}; 


\addplot [draw=none, mark=none, name path=random_low] table [x=episode, y expr=\thisrowno{11}+\thisrowno{12}, col sep=comma] {data/per_ep_regret_kai_n15_b10_s2_a2_H20_L30_epochs1.csv};
\addplot [draw=none, mark=none, name path=random_high] table [x=episode, y expr=\thisrowno{11}-\thisrowno{12}, col sep=comma] {data/per_ep_regret_kai_n15_b10_s2_a2_H20_L30_epochs1.csv};
\tikzfillbetween[of=random_low and random_high, on layer=main]{gray, opacity=0.1};

\addplot [draw=none, mark=none, name path=wiql_low] table [x=episode, y expr=\thisrowno{9}+\thisrowno{10}, col sep=comma] {data/per_ep_regret_kai_n15_b10_s2_a2_H20_L30_epochs1.csv};
\addplot [draw=none, mark=none, name path=wiql_high] table [x=episode, y expr=\thisrowno{9}-\thisrowno{10}, col sep=comma] {data/per_ep_regret_kai_n15_b10_s2_a2_H20_L30_epochs1.csv};
\tikzfillbetween[of=wiql_low and wiql_high, on layer=main]{BurntOrange, opacity=0.1};

\addplot [draw=none, mark=none, name path=ucw_extreme_low] table [x=episode, y expr=\thisrowno{7}+\thisrowno{8}, col sep=comma] {data/per_ep_regret_kai_n15_b10_s2_a2_H20_L30_epochs1.csv};
\addplot [draw=none, mark=none, name path=ucw_extreme_high] table [x=episode, y expr=\thisrowno{7}-\thisrowno{8}, col sep=comma] {data/per_ep_regret_kai_n15_b10_s2_a2_H20_L30_epochs1.csv};
\tikzfillbetween[of=ucw_extreme_low and ucw_extreme_high, on layer=main]{ForestGreen, opacity=0.1};

\addplot [draw=none, mark=none, name path=ucw_subsidy_low] table [x=episode, y expr=\thisrowno{5}+\thisrowno{6}, col sep=comma] {data/per_ep_regret_kai_n15_b10_s2_a2_H20_L30_epochs1.csv};
\addplot [draw=none, mark=none, name path=ucw_subsidy_high] table [x=episode, y expr=\thisrowno{5}-\thisrowno{6}, col sep=comma] {data/per_ep_regret_kai_n15_b10_s2_a2_H20_L30_epochs1.csv};
\tikzfillbetween[of=ucw_subsidy_low and ucw_subsidy_high, on layer=main]{cyan, opacity=0.1};

\addplot [draw=none, mark=none, name path=ucw_value_low] table [x=episode, y expr=\thisrowno{3}+\thisrowno{4}, col sep=comma] {data/per_ep_regret_kai_n15_b10_s2_a2_H20_L30_epochs1.csv};
\addplot [draw=none, mark=none, name path=ucw_value_high] table [x=episode, y expr=\thisrowno{3}-\thisrowno{4}, col sep=comma] {data/per_ep_regret_kai_n15_b10_s2_a2_H20_L30_epochs1.csv};
\tikzfillbetween[of=ucw_value_low and ucw_value_high, on layer=main]{blue, opacity=0.05};
\end{axis}

\end{tikzpicture}

%% file: figures/plot_episode_len.tex
\begin{tikzpicture}
\pgfplotsset{
  width=0.28\linewidth,
  height=0.2\linewidth,
  ytick pos=left,
  tick label style={font=\scriptsize},
  ymajorgrids=true,
  xtick style={draw=none},
}

\begin{axis}[
  at={(0.24\linewidth, 0\linewidth)}, 
  table/col sep=comma,
  xlabel={\footnotesize{(b)~$H=10$}},
  ylabel style={align=center},
  xmin=0, xmax=120,
  error bars/y dir=both, 
  error bars/y explicit,  
  error bars/error bar style={color=black, thick},
]

\addplot[blue, line width=1pt, mark=none] table [x=episode, y=ucw_value_avg] {data/per_ep_regret_kai_n8_b3_s2_a2_H10_L120_epochs1.csv}; 
\addplot[cyan, line width=1pt, mark=none] table [x=episode, y=ucw_qp_avg] {data/per_ep_regret_kai_n8_b3_s2_a2_H10_L120_epochs1.csv}; 
\addplot[ForestGreen, densely dashdotted, line width=1pt, mark=none] table [x=episode, y=ucw_extreme_avg] {data/per_ep_regret_kai_n8_b3_s2_a2_H10_L120_epochs1.csv}; 
\addplot[BurntOrange, densely dashed, line width=1pt, mark=none] table [x=episode, y=wiql_avg] {data/per_ep_regret_kai_n8_b3_s2_a2_H10_L120_epochs1.csv}; 
\addplot[gray, densely dotted, line width=.7pt, mark=none] table [x=episode, y=random_avg] {data/per_ep_regret_kai_n8_b3_s2_a2_H10_L120_epochs1.csv}; 


\addplot [draw=none, mark=none, name path=random_low] table [x=episode, y expr=\thisrowno{11}+\thisrowno{12}, col sep=comma] {data/per_ep_regret_kai_n8_b3_s2_a2_H10_L120_epochs1.csv};
\addplot [draw=none, mark=none, name path=random_high] table [x=episode, y expr=\thisrowno{11}-\thisrowno{12}, col sep=comma] {data/per_ep_regret_kai_n8_b3_s2_a2_H10_L120_epochs1.csv};
\tikzfillbetween[of=random_low and random_high, on layer=main]{gray, opacity=0.1};

\addplot [draw=none, mark=none, name path=wiql_low] table [x=episode, y expr=\thisrowno{9}+\thisrowno{10}, col sep=comma] {data/per_ep_regret_kai_n8_b3_s2_a2_H10_L120_epochs1.csv};
\addplot [draw=none, mark=none, name path=wiql_high] table [x=episode, y expr=\thisrowno{9}-\thisrowno{10}, col sep=comma] {data/per_ep_regret_kai_n8_b3_s2_a2_H10_L120_epochs1.csv};
\tikzfillbetween[of=wiql_low and wiql_high, on layer=main]{BurntOrange, opacity=0.1};

\addplot [draw=none, mark=none, name path=ucw_extreme_low] table [x=episode, y expr=\thisrowno{7}+\thisrowno{8}, col sep=comma] {data/per_ep_regret_kai_n8_b3_s2_a2_H10_L120_epochs1.csv};
\addplot [draw=none, mark=none, name path=ucw_extreme_high] table [x=episode, y expr=\thisrowno{7}-\thisrowno{8}, col sep=comma] {data/per_ep_regret_kai_n8_b3_s2_a2_H10_L120_epochs1.csv};
\tikzfillbetween[of=ucw_extreme_low and ucw_extreme_high, on layer=main]{ForestGreen, opacity=0.1};

\addplot [draw=none, mark=none, name path=ucw_subsidy_low] table [x=episode, y expr=\thisrowno{5}+\thisrowno{6}, col sep=comma] {data/per_ep_regret_kai_n8_b3_s2_a2_H10_L120_epochs1.csv};
\addplot [draw=none, mark=none, name path=ucw_subsidy_high] table [x=episode, y expr=\thisrowno{5}-\thisrowno{6}, col sep=comma] {data/per_ep_regret_kai_n8_b3_s2_a2_H10_L120_epochs1.csv};
\tikzfillbetween[of=ucw_subsidy_low and ucw_subsidy_high, on layer=main]{cyan, opacity=0.1};

\addplot [draw=none, mark=none, name path=ucw_value_low] table [x=episode, y expr=\thisrowno{3}+\thisrowno{4}, col sep=comma] {data/per_ep_regret_kai_n8_b3_s2_a2_H10_L120_epochs1.csv};
\addplot [draw=none, mark=none, name path=ucw_value_high] table [x=episode, y expr=\thisrowno{3}-\thisrowno{4}, col sep=comma] {data/per_ep_regret_kai_n8_b3_s2_a2_H10_L120_epochs1.csv};
\tikzfillbetween[of=ucw_value_low and ucw_value_high, on layer=main]{blue, opacity=0.05};
\end{axis}

\begin{axis}[
  at={(0\linewidth, .0\linewidth)}, 
  table/col sep=comma,
  xlabel={\footnotesize{(a)~$H=5$}},
  ylabel={\footnotesize{Regret}},
   xmin=0, xmax=240,
  error bars/y dir=both, 
  error bars/y explicit,  
  error bars/error bar style={color=black, thick},
]

\addplot[blue, line width=1pt, mark=none] table [x=episode, y=ucw_value_avg] {data/per_ep_regret_kai_n8_b3_s2_a2_H5_L240_epochs1.csv}; 
\addplot[cyan, line width=1pt, mark=none] table [x=episode, y=ucw_qp_avg] {data/per_ep_regret_kai_n8_b3_s2_a2_H5_L240_epochs1.csv}; 
\addplot[ForestGreen, densely dashdotted, line width=1pt, mark=none] table [x=episode, y=ucw_extreme_avg] {data/per_ep_regret_kai_n8_b3_s2_a2_H5_L240_epochs1.csv}; 
\addplot[BurntOrange, densely dashed, line width=1pt, mark=none] table [x=episode, y=wiql_avg] {data/per_ep_regret_kai_n8_b3_s2_a2_H5_L240_epochs1.csv}; 
\addplot[gray, densely dotted, line width=.7pt, mark=none] table [x=episode, y=random_avg] {data/per_ep_regret_kai_n8_b3_s2_a2_H5_L240_epochs1.csv}; 


\addplot [draw=none, mark=none, name path=random_low] table [x=episode, y expr=\thisrowno{11}+\thisrowno{12}, col sep=comma] {data/per_ep_regret_kai_n8_b3_s2_a2_H5_L240_epochs1.csv};
\addplot [draw=none, mark=none, name path=random_high] table [x=episode, y expr=\thisrowno{11}-\thisrowno{12}, col sep=comma] {data/per_ep_regret_kai_n8_b3_s2_a2_H5_L240_epochs1.csv};
\tikzfillbetween[of=random_low and random_high, on layer=main]{gray, opacity=0.1};

\addplot [draw=none, mark=none, name path=wiql_low] table [x=episode, y expr=\thisrowno{9}+\thisrowno{10}, col sep=comma] {data/per_ep_regret_kai_n8_b3_s2_a2_H5_L240_epochs1.csv};
\addplot [draw=none, mark=none, name path=wiql_high] table [x=episode, y expr=\thisrowno{9}-\thisrowno{10}, col sep=comma] {data/per_ep_regret_kai_n8_b3_s2_a2_H5_L240_epochs1.csv};
\tikzfillbetween[of=wiql_low and wiql_high, on layer=main]{BurntOrange, opacity=0.1};

\addplot [draw=none, mark=none, name path=ucw_extreme_low] table [x=episode, y expr=\thisrowno{7}+\thisrowno{8}, col sep=comma] {data/per_ep_regret_kai_n8_b3_s2_a2_H5_L240_epochs1.csv};
\addplot [draw=none, mark=none, name path=ucw_extreme_high] table [x=episode, y expr=\thisrowno{7}-\thisrowno{8}, col sep=comma] {data/per_ep_regret_kai_n8_b3_s2_a2_H5_L240_epochs1.csv};
\tikzfillbetween[of=ucw_extreme_low and ucw_extreme_high, on layer=main]{ForestGreen, opacity=0.1};

\addplot [draw=none, mark=none, name path=ucw_subsidy_low] table [x=episode, y expr=\thisrowno{5}+\thisrowno{6}, col sep=comma] {data/per_ep_regret_kai_n8_b3_s2_a2_H5_L240_epochs1.csv};
\addplot [draw=none, mark=none, name path=ucw_subsidy_high] table [x=episode, y expr=\thisrowno{5}-\thisrowno{6}, col sep=comma] {data/per_ep_regret_kai_n8_b3_s2_a2_H5_L240_epochs1.csv};
\tikzfillbetween[of=ucw_subsidy_low and ucw_subsidy_high, on layer=main]{cyan, opacity=0.1};

\addplot [draw=none, mark=none, name path=ucw_value_low] table [x=episode, y expr=\thisrowno{3}+\thisrowno{4}, col sep=comma] {data/per_ep_regret_kai_n8_b3_s2_a2_H5_L240_epochs1.csv};
\addplot [draw=none, mark=none, name path=ucw_value_high] table [x=episode, y expr=\thisrowno{3}-\thisrowno{4}, col sep=comma] {data/per_ep_regret_kai_n8_b3_s2_a2_H5_L240_epochs1.csv};
\tikzfillbetween[of=ucw_value_low and ucw_value_high, on layer=main]{blue, opacity=0.05};

\end{axis}

\begin{axis}[
  at={(0.48\linewidth, .0\linewidth)}, 
  table/col sep=comma,
  xlabel={\footnotesize{(c)~$H=30$}},
  xmin=0, xmax=39,
  error bars/y dir=both, 
  error bars/y explicit,  
  error bars/error bar style={color=black, thick},
]

\addplot[blue, line width=1pt, mark=none] table [x=episode, y=ucw_value_avg] {data/per_ep_regret_kai_n8_b3_s2_a2_H30_L40_epochs1.csv}; 
\addplot[cyan, line width=1pt, mark=none] table [x=episode, y=ucw_qp_avg] {data/per_ep_regret_kai_n8_b3_s2_a2_H30_L40_epochs1.csv}; 
\addplot[ForestGreen, densely dashdotted, line width=1pt, mark=none] table [x=episode, y=ucw_extreme_avg] {data/per_ep_regret_kai_n8_b3_s2_a2_H30_L40_epochs1.csv}; 
\addplot[BurntOrange, densely dashed, line width=1pt, mark=none] table [x=episode, y=wiql_avg] {data/per_ep_regret_kai_n8_b3_s2_a2_H30_L40_epochs1.csv}; 
\addplot[gray, densely dotted, line width=.7pt, mark=none] table [x=episode, y=random_avg] {data/per_ep_regret_kai_n8_b3_s2_a2_H30_L40_epochs1.csv}; 


\addplot [draw=none, mark=none, name path=random_low] table [x=episode, y expr=\thisrowno{11}+\thisrowno{12}, col sep=comma] {data/per_ep_regret_kai_n8_b3_s2_a2_H30_L40_epochs1.csv};
\addplot [draw=none, mark=none, name path=random_high] table [x=episode, y expr=\thisrowno{11}-\thisrowno{12}, col sep=comma] {data/per_ep_regret_kai_n8_b3_s2_a2_H30_L40_epochs1.csv};
\tikzfillbetween[of=random_low and random_high, on layer=main]{gray, opacity=0.1};

\addplot [draw=none, mark=none, name path=wiql_low] table [x=episode, y expr=\thisrowno{9}+\thisrowno{10}, col sep=comma] {data/per_ep_regret_kai_n8_b3_s2_a2_H30_L40_epochs1.csv};
\addplot [draw=none, mark=none, name path=wiql_high] table [x=episode, y expr=\thisrowno{9}-\thisrowno{10}, col sep=comma] {data/per_ep_regret_kai_n8_b3_s2_a2_H30_L40_epochs1.csv};
\tikzfillbetween[of=wiql_low and wiql_high, on layer=main]{BurntOrange, opacity=0.1};

\addplot [draw=none, mark=none, name path=ucw_extreme_low] table [x=episode, y expr=\thisrowno{7}+\thisrowno{8}, col sep=comma] {data/per_ep_regret_kai_n8_b3_s2_a2_H30_L40_epochs1.csv};
\addplot [draw=none, mark=none, name path=ucw_extreme_high] table [x=episode, y expr=\thisrowno{7}-\thisrowno{8}, col sep=comma] {data/per_ep_regret_kai_n8_b3_s2_a2_H30_L40_epochs1.csv};
\tikzfillbetween[of=ucw_extreme_low and ucw_extreme_high, on layer=main]{ForestGreen, opacity=0.1};

\addplot [draw=none, mark=none, name path=ucw_subsidy_low] table [x=episode, y expr=\thisrowno{5}+\thisrowno{6}, col sep=comma] {data/per_ep_regret_kai_n8_b3_s2_a2_H30_L40_epochs1.csv};
\addplot [draw=none, mark=none, name path=ucw_subsidy_high] table [x=episode, y expr=\thisrowno{5}-\thisrowno{6}, col sep=comma] {data/per_ep_regret_kai_n8_b3_s2_a2_H30_L40_epochs1.csv};
\tikzfillbetween[of=ucw_subsidy_low and ucw_subsidy_high, on layer=main]{cyan, opacity=0.1};

\addplot [draw=none, mark=none, name path=ucw_value_low] table [x=episode, y expr=\thisrowno{3}+\thisrowno{4}, col sep=comma] {data/per_ep_regret_kai_n8_b3_s2_a2_H30_L40_epochs1.csv};
\addplot [draw=none, mark=none, name path=ucw_value_high] table [x=episode, y expr=\thisrowno{3}-\thisrowno{4}, col sep=comma] {data/per_ep_regret_kai_n8_b3_s2_a2_H30_L40_epochs1.csv};
\tikzfillbetween[of=ucw_value_low and ucw_value_high, on layer=main]{blue, opacity=0.05};
\end{axis}

\begin{axis}[
  at={(0.72\linewidth, .0\linewidth)}, 
  table/col sep=comma,
  xlabel={\footnotesize{(d)~$H=50$}},
  xmin=0, xmax=23,
  error bars/y dir=both, 
  error bars/y explicit,  
  error bars/error bar style={color=black, thick},
]

\addplot[blue, line width=1pt, mark=none] table [x=episode, y=ucw_value_avg] {data/per_ep_regret_kai_n8_b3_s2_a2_H50_L24_epochs1.csv}; 
\addplot[cyan, line width=1pt, mark=none] table [x=episode, y=ucw_qp_avg] {data/per_ep_regret_kai_n8_b3_s2_a2_H50_L24_epochs1.csv}; 
\addplot[ForestGreen, densely dashdotted, line width=1pt, mark=none] table [x=episode, y=ucw_extreme_avg] {data/per_ep_regret_kai_n8_b3_s2_a2_H50_L24_epochs1.csv}; 
\addplot[BurntOrange, densely dashed, line width=1pt, mark=none] table [x=episode, y=wiql_avg] {data/per_ep_regret_kai_n8_b3_s2_a2_H50_L24_epochs1.csv}; 
\addplot[gray, densely dotted, line width=.7pt, mark=none] table [x=episode, y=random_avg] {data/per_ep_regret_kai_n8_b3_s2_a2_H50_L24_epochs1.csv}; 


\addplot [draw=none, mark=none, name path=random_low] table [x=episode, y expr=\thisrowno{11}+\thisrowno{12}, col sep=comma] {data/per_ep_regret_kai_n8_b3_s2_a2_H50_L24_epochs1.csv};
\addplot [draw=none, mark=none, name path=random_high] table [x=episode, y expr=\thisrowno{11}-\thisrowno{12}, col sep=comma] {data/per_ep_regret_kai_n8_b3_s2_a2_H50_L24_epochs1.csv};
\tikzfillbetween[of=random_low and random_high, on layer=main]{gray, opacity=0.1};

\addplot [draw=none, mark=none, name path=wiql_low] table [x=episode, y expr=\thisrowno{9}+\thisrowno{10}, col sep=comma] {data/per_ep_regret_kai_n8_b3_s2_a2_H50_L24_epochs1.csv};
\addplot [draw=none, mark=none, name path=wiql_high] table [x=episode, y expr=\thisrowno{9}-\thisrowno{10}, col sep=comma] {data/per_ep_regret_kai_n8_b3_s2_a2_H50_L24_epochs1.csv};
\tikzfillbetween[of=wiql_low and wiql_high, on layer=main]{BurntOrange, opacity=0.1};

\addplot [draw=none, mark=none, name path=ucw_extreme_low] table [x=episode, y expr=\thisrowno{7}+\thisrowno{8}, col sep=comma] {data/per_ep_regret_kai_n8_b3_s2_a2_H50_L24_epochs1.csv};
\addplot [draw=none, mark=none, name path=ucw_extreme_high] table [x=episode, y expr=\thisrowno{7}-\thisrowno{8}, col sep=comma] {data/per_ep_regret_kai_n8_b3_s2_a2_H50_L24_epochs1.csv};
\tikzfillbetween[of=ucw_extreme_low and ucw_extreme_high, on layer=main]{ForestGreen, opacity=0.1};

\addplot [draw=none, mark=none, name path=ucw_subsidy_low] table [x=episode, y expr=\thisrowno{5}+\thisrowno{6}, col sep=comma] {data/per_ep_regret_kai_n8_b3_s2_a2_H50_L24_epochs1.csv};
\addplot [draw=none, mark=none, name path=ucw_subsidy_high] table [x=episode, y expr=\thisrowno{5}-\thisrowno{6}, col sep=comma] {data/per_ep_regret_kai_n8_b3_s2_a2_H50_L24_epochs1.csv};
\tikzfillbetween[of=ucw_subsidy_low and ucw_subsidy_high, on layer=main]{cyan, opacity=0.1};

\addplot [draw=none, mark=none, name path=ucw_value_low] table [x=episode, y expr=\thisrowno{3}+\thisrowno{4}, col sep=comma] {data/per_ep_regret_kai_n8_b3_s2_a2_H50_L24_epochs1.csv};
\addplot [draw=none, mark=none, name path=ucw_value_high] table [x=episode, y expr=\thisrowno{3}-\thisrowno{4}, col sep=comma] {data/per_ep_regret_kai_n8_b3_s2_a2_H50_L24_epochs1.csv};
\tikzfillbetween[of=ucw_value_low and ucw_value_high, on layer=main]{blue, opacity=0.05};
\end{axis}

\end{tikzpicture}

%% file: appendix.tex
\onecolumn
\appendix


\section{Notation}
All the notations used in the problem statement, restless multi-armed bandits, and regret analysis are shown in Table~\ref{table:notation-problem} and Table~\ref{table:notation-rmab}.

\begin{table}[!ht]
    \centering
    \begin{tabular}{ll}
\toprule
\multicolumn{2}{c}{\textit{Problem instantiation}} \\
\midrule
\midrule
\textbf{Symbol} & \textbf{Definition} \\ \midrule
$K$ & Budget in each timestep \\
$N$ & Number of arms. Each arm indexed by $i \in [N]$ \\
$\round$ & Episode \\
$\Round$ & Number of episodes \\
$\horizon$ & Timestep within a single episode \\
$\Horizon$ & Horizon length for a single episode \\
$\gamma$ & Discount factor, with $\gamma \in (0, 1)$ \\
\bottomrule
    \end{tabular}
\caption{List of common notations in the problem statement}
\label{table:notation-problem}
\end{table}

\begin{table}[!ht]
    \centering
    \begin{tabular}{ll}
\toprule
\multicolumn{2}{c}{\textit{Restless bandit notation}} \\
\midrule
\midrule
\textbf{Symbol} & \textbf{Definition} \\
\midrule
$\transitions$ & Set of transition probabilities across all arms, with $\transition_i$ as transitions for a single arm \\
$\truep$ & True transition probabilities \\
$\stateset$ & Set of finitely many possible states with $|\stateset| = \nstates$ possible states \\
$\statevec_\horizon$ & State of the RMAB instance at timestep $\horizon$, with $\statevec_\horizon \in \stateset^N$ and initial state $\statevec_{\text{init}}$\\
$\state_{\horizon,i}$ & State of arm $i \in [N]$ at timestep $\horizon$ \\
$\Action$ & Set of possible actions. We consider $\{0, 1\}$ \\
$\actionvec_\horizon$ & Action at time $\horizon$, with $\actionvec_\horizon \in \Action^N$ \\
$\action_{\horizon,i}$ & Action taken on arm $i$ at timestep $\horizon$\\
$R$ & Given reward function as a function of the state and action $R: \stateset \times \Action \rightarrow \R$. \\
$\policy^{(\round)}$ & Learner's policy in episode~$\round$, where $\policy^{(\round)} : \stateset^N \rightarrow \Action^N$\\
$\optpolicy$ & The optimal policy that maximizes the total future reward. \\
$\maxPenalty$ & The optimization problem defined to maximize the optimistic Whittle index value. \\
$\maxValue$ & The optimization problem defined to maximize the optimistic future value. \\
$Q^{m_i}(s,a)$ & Q-value in Bellman equation. The Q-value is defined as the future value associated to the current state and action. \\
$R(s_{\horizon,i}, a_{h,i})$ & Reward from arm $i$ at timestep $\horizon$ with action $a_{\horizon,i}$\\
$U^{\transitions, \lambda}_{\policy}(\statevec_{1})$ & Lagrangian relaxation of learner's objective, with optimal value $U^{\transitions, \lambda}_\star$ \\
$V^{P_i, \lambda}_{\policy_i}(s_{1,i})$ & Value for being in state~$s_i$ \\
$\lambda$ & Global penalty for taking action $a = 1$ \\
$\whittle_i$ & Whittle index penalty for arm $i$ \\
$W_i(P_i, s_i)$ & Whittle index of arm $i$ with transitions $P_i$ and state $s_i$ \\
$\dataset$ & Dataset of historical transitions \\
\bottomrule
    \end{tabular}
\caption{List of common notations in the RMAB regret analysis}
\label{table:notation-rmab}
\end{table}

\section{Societal Impacts}
\label{sec:societal-impact}
Restless bandits have been increasingly applied to socially impactful problems including healthcare and energy distribution. In these settings, we would likely not know the transition dynamics in advance, particularly if we are working with a new patient population (for healthcare) or new residential community (for energy). 
Even past work on streaming bandits \citep{mate2022efficient} which allow for new mothers to enroll over time assume that the transition probabilities are fully known in advance, which is not realistic.
Our UCWhittle approach enabling online learning for RMABs has the potential to greatly broaden the applicability of RMABs for social impact, particularly as our theoretical results guarantee limited regret.






\section{Limitations}
One challenge with our UCWhittle approach is that online learning often converges slower than offline learning that reuses all the data to train for many epochs. In order to accommodate new data coming in, online learning approaches often take a single update when each new data arrives. In contrast, offline learning can iterate through the same data for many times, which allows offline learning approaches to fit the data repeatedly. Therefore, online learning approaches often require more data to reach the same performance as offline learning approaches.

However, this slower learning behavior also allows online learning approaches to be less biased to the existing dataset. Online learning approaches are incentivized to explore and update data that is less queried previously, which also encourages exploring underrepresented groups. This property encourages the exploration process and reduce bias to the learned model. This is particularly important when there are features involved in the learning process. Online learning approaches are able to explore unseen features more, while offline learning approaches often rely on extrapolation and are unable to handle unseen features.
Our work further extends research in online learning in RMABs, which also helps explore more possibility to accommodate new data and new features that are unseen in the existing dataset.

\section{ARMMAN: Maternal and Child Health Data}\label{sec:consent}


In the maternal mobile health program operated by ARMMAN, each instance is composed of a set of mothers who participated in the program for $10$ or more weeks. The dataset contains the states of each beneficiary, actions on whether a service call was scheduled to the beneficiary or not, and the beneficiary's next states after receiving the calling (or not calling) actions. This dataset is used to construct a set of empirical estimates of the transition probabilities and build an interactive, simulated RMAB environment. Our online learning algorithm then interacts with the environment to learn the transition probabilities and optimize total engagement. The experiments in this paper were all done in simulation.

Specifically, this problem is modelled as a $2$-state (Engaging and Non-Engaging) RMAB problem where we do not know each beneficiary's transition behavior --- transition between Engaging and Non-Engaging state, determined by whether the beneficiary listens to an automated voice message (average length 1 minute) for more than 30 seconds. The goal of the online learning challenge is to simultaneously learn the missing transition probabilities and optimize the overall engagement of all mothers under budget constraints.
The ARMMAN data is also abstracted out and contains no personally identifiable information or demographic features related to the mothers.

In the following sections, we provide more detailed information about consent related to data collection, analyzing data, data usage and sharing.

\subsection{Secondary Analysis and Data Usage}
This study falls into the category of secondary analysis of the aforementioned dataset shared by ARMMAN. We randomly sampled from the previously collected engagement probabilities of different mothers participating in the service call program to simulate online learning environment.
This paper does not involve deployment of the proposed algorithm or any other baselines to the service call program. As noted earlier, the experiments are secondary analysis with approval from the ARMMAN ethics board.

\subsection{Consent for Data Collection and Sharing}
The consent for collecting data is obtained from each of the participants of the service call program. The data collection process is carefully explained to the participants to seek their consent before collecting the data. The data is anonymized before sharing with us to ensure anonymity. 
Data exchange and use was regulated through clearly defined exchange protocols including anonymization, read-access only to researchers, restricted use of the data for research purposes only, and approval by ARMMAN's ethics review committee.

\subsection{Universal Accessibility of Health Information}
To allay further concerns: this simulation study focuses on improving quality of service calls. Even in the intended future application, all participants will receive the same weekly health information by automated message regardless of whether they are scheduled to receive service calls or not.  The service call program does not withhold any information from the participants nor conduct any experimentation on the health information. The health information is always available to all participants, and participants can always request service calls via a free missed call service. In the intended future application our algorithm may only help schedule {\it additional} service calls to help mothers who are likely to drop out of the program.

\section{Full Proofs}

\label{sec:full-proofs}

\subsection{Confidence Bound}
\chernoffBound*
\begin{proof}
Generally, the L1-deviation of the true distribution and the empirical distribution over $m$ distinct events from $n$ samples is bounded according to \cite{weissman2003inequalities} by:
\begin{align}
    \Pr(\norm{\hat{p} - p}_1 \geq \epsilon) \leq (2^{m} - 2) \exp^{(-\frac{n\epsilon^2}{2})}
\end{align}
This result can be applied to our case to compare $P^{(\round)}_i(s,a,\cdot) \in \R^{|S|}$ with $P^\star(s,a,\cdot) \in \R^{|S|}$ for every state $s$ and action $a$. We have:
\begin{align}
    \Pr \left(\norm{P^{(\round)}_i(s,a,\cdot) - P^\star(s,a,\cdot)}_1 \geq \epsilon \right) \leq \left( 2^{|S|} - 2 \right) \exp^{\left(-\frac{n\epsilon^2}{2} \right)}
\end{align}
By choosing $\epsilon = \sqrt{\frac{2}{n} \log \left(2^{|S|} |S| |A| N \frac{\round^4}{\delta} \right)} \leq  \sqrt{\frac{2 |S|}{n} \log \left( 2 |S| |A| N \frac{\round^4}{\delta} \right)}$, we have:
\begin{align}
    \Pr \left(\norm{P^{(\round)}_i(s,a,\cdot) - P^\star(s,a,\cdot)}_1 \geq \sqrt{\frac{2 |S|}{n} \log \left( 2 |S| |A| N \frac{\round^4}{\delta} \right)} \right) \leq&~ 2^{|S|} \exp^{- \log \left( 2^{|S|} |S| |A| N \frac{\round^4}{\delta} \right)} \nonumber \\
    =&~ \frac{\delta}{|S| |A| N \round^4}
\end{align}
Set $n = \max \{1, N_i^{(\round)}(s,a) \}$ for each pair of $(s,a)$.
Taking union bound over all states $s \in \stateset$, actions $a \in \Action$, and arms $i \in [N]$ yields:
\begin{align}
    \Pr \left(\truep \not\in \boldsymbol B^{(\round)} \right) \leq\frac{\delta}{\round^4} \quad \Longrightarrow \quad \Pr \left(\truep \in \boldsymbol B^{(\round)} \right) \geq 1- \frac{\delta}{\round^4}
\end{align}
\end{proof}

\subsection{Regret Decomposition}

\perEpisode*
\begin{proof}
Since the value function is a fixed point to the corresponding Bellman operator, we have:
\begin{align}
        V^{P^{(\round)}_i}_{\policy^{(\round)}_i}(s_{1,i}) - V^{P^\star_i}_{\policy^{(\round)}_i}(s_{1,i}) &= \left( \mathcal{T}^{P^{(\round)}_i}_{\policy^{(\round)}_i} V^{P^{(\round)}_i}_{\policy^{(\round)}_i} - \mathcal{T}^{P^\star_i}_{\policy^{(\round)}_i} V^{P^\star_i}_{\policy^{(\round)}_i} \right)(s_{1,i}) \nonumber \\
        &= \left(\mathcal{T}^{P^{(\round)}_i}_{\policy^{(\round)}_i} - \mathcal{T}^{P^\star_i}_{\policy^{(\round)}_i} \right) V^{P^{(\round)}_i}_{\policy^{(\round)}_i}(s_{1,i}) + \mathcal{T}^{P^\star_i}_{\policy^{(\round)}_i} \left( V^{P^{(\round)}_i}_{\policy^{(\round)}_i} - V^{P^\star_i}_{\policy^{(\round)}_i} \right)(s_{1,i}) \label{eqn:regret-decomposition-bellman}
\end{align}
where the second term in Equation~\eqref{eqn:regret-decomposition-bellman} can be further expanded by the Bellman operator:
\begin{align}
    \mathcal{T}^{P^\star_i}_{\policy^{(\round)}_i} (V^{P^{(\round)}_i}_{\policy^{(\round)}_i} - V^{P^\star_i}_{\policy^{(\round)}_i})(s_{1,i}) =& ~\mathbb{E}_{a \sim \policy^{(\round)}_i} \left[ \gamma \sum\nolimits_{s' \in \stateset} P^\star_i(s_{1,i}, a, s') (V^{P_i^{(\round)}}_{\policy_i^{(\round)}}(s') - V^{P_i^\star}_{\policy_i^{(\round)}}(s')) \right] \nonumber \\
    =& ~\gamma \mathbb{E}_{s_{2,i} \sim P^\star_i, \policy^{(\round)}_i} \left[ V^{P^{(\round)}_i}_{\policy^{(\round)}_i}(s_{2,i}) - V^{P^\star_i}_{\policy^{(\round)}_i}(s_{2,i}) \right] \label{eqn:regret-decomposition-bellman2}
\end{align}
We can repeatedly apply the decomposition process in Equation~\eqref{eqn:regret-decomposition-bellman} to the value function difference in Equation~\eqref{eqn:regret-decomposition-bellman2} to get Equation~\eqref{eqn:regret-decomposition}, which concludes the proof.
\end{proof}

\subsection{Regret Bound with Given Penalty}
\regretBound*
\begin{proof}
We can write
\begin{align}
    \text{Reg}(T) =& \sum\nolimits_{\round=1}^\Round \text{Reg}^{(\round)} = \sum\nolimits_{\round=1}^\Round \left( \text{Reg}^{(\round)} \mathds{1}_{\truep \not\in \boldsymbol B^{(\round)}} + \text{Reg}^{(\round)} \mathds{1}_{\truep \in \boldsymbol B^{(\round)}} \right) \nonumber \\
    =& \sum\nolimits_{\round=1}^\Round \text{Reg}^{(\round)} \mathds{1}_{\truep \not\in \boldsymbol B^{(\round)}} + \sum\nolimits_{\round=1}^\Round \text{Reg}^{(\round)} \mathds{1}_{\truep \in \boldsymbol B^{(\round)}}
\end{align}
We will analyze both terms separately and combine them together in the end.

\paragraph{Regret when the confidence bounds do not hold}
\begin{align}
    \sum\nolimits_{\round=1}^\Round \text{Reg}^{(\round)} \mathds{1}_{\truep \not\in \boldsymbol B^{(\round)}} =& \sum\nolimits_{\round=1}^{\sqrt{\Round}} \text{Reg}^{(\round)} \mathds{1}_{\truep \not\in \boldsymbol B^{(\round)}} + \sum\nolimits_{\round=\sqrt{\Round}+1}^\Round \text{Reg}^{(\round)} \mathds{1}_{\truep \not\in \boldsymbol B^{(\round)}} \nonumber \\
    \leq & \frac{N R_{\text{max}}}{1 - \gamma} \sqrt{\Round} + \sum\nolimits_{\round=\sqrt{\Round}+1}^\Round \text{Reg}^{(\round)} \mathds{1}_{\truep \not\in \boldsymbol B^{(\round)}}
\end{align}
where we use the trivial upper bound of the individual regret $\text{Reg}^{(\round)} \leq \frac{N R_{\text{max}}}{1 - \gamma}$ for all $\round$, where $R_{\text{max}}$ is the maximal reward per time step.

Notice that the second term vanishes with probability:
\begin{align}
\Pr\left( \left\{ \truep \in \boldsymbol B^{(\round)} ~\forall \sqrt{\Round} \leq \round \leq \Round \right\} \right) \geq & ~ 1 - \sum\nolimits_{\sqrt{\Round} \leq \round \leq \Round} \Pr\left( \left\{ \truep \in \boldsymbol B^{(\round)} \right\} \right) \nonumber \\
\geq & ~ 1 - \sum\nolimits_{\sqrt{\Round} \leq \round \leq \Round} \frac{\delta}{\round^4} \nonumber \\
\geq & ~ 1 - \sum\nolimits_{\sqrt{\Round} \leq \round \leq \Round} \frac{3 \delta}{\round^4} \nonumber \\
\geq & ~ 1 - \int_{\sqrt{\Round}}^\infty \frac{3 \delta}{\round^4} d\round \nonumber \\
= & ~ 1 - \frac{\delta}{\Round^{3/2}}
\end{align}
Therefore, the regret outside of confidence bounds is upper bounded by $O(\sqrt{\Round})$ with probability at least $1 - \frac{\delta}{\Round^{3/2}}$.
We can apply union bound to all possible $\Round \in \N$, which holds with high probability:
\begin{align}
    1 - \sum\nolimits_{\Round=1}^\infty \frac{\delta}{\Round^{3/2}} = 1 - O(\delta) \ .
\end{align}

\paragraph{Regret when the confidence bounds hold}
Notice that
\begin{align}
    &\left( \mathcal{T}_{\policy^{(\round)}_i}^{P^{(\round)}_i} - \mathcal{T}_{\policy^{(\round)}_i}^{P^\star_i} \right) V(s) \nonumber \\
    =& \mathop{\mathbb{E}}_{a \sim \policy^{(\round)}_i} \left[ \left( R(s,a) + \sum\nolimits_{s' \in \stateset} P^{(\round)}_i(s,a,s') V(s') \right) - \left( R(s,a) + \sum\nolimits_{s' \in \stateset} P^\star_i(s,a,s') \right) V(s') \right] \nonumber \\
    =& \mathop\mathbb{E}_{a \sim \policy^{(\round)}_i} \left[ \sum\nolimits_{s' \in \stateset} (P^{(\round)}_i(s,a,s') - P^\star_i(s,a,s')) V(s') \right] \nonumber
\end{align}

When the confidence bound holds $\truep \in \boldsymbol B^{(\round)}$, we can bound the regret at round $l$ by:
\begin{align}
    \text{Reg}^{(\round)} =& ~U^{\boldsymbol P^{(\round)}}_{\policy^{(\round)}}(\boldsymbol s_1) - U^{\truep}_{\policy^{(\round)}}(\boldsymbol s_1) \nonumber \\
    =& \sum\nolimits_{i=1}^N V^{P^{(\round)}_i}_{\policy^{(\round)}_i}(s_{1,i}) - V^{P^\star_i}_{\policy^{(\round)}_i}(s_{1,i}) \nonumber \\
    =& \sum\nolimits_{i=1}^N \mathbb{E}_{P^\star_i, \policy^{(\round)}_i} \left[ \sum\nolimits_{h=1}^{\infty} \gamma^{h-1} (\mathcal{T}^{P^{(\round)}_i}_{\policy^{(\round)}_i} - \mathcal{T}^{P^\star_i}_{\policy^{(\round)}_i}) V^{P^{(\round)}_i}_{\policy^{(\round)}_i}(s_{h,i}) \right] \nonumber \\
    =& \sum\nolimits_{i=1}^N \mathbb{E}_{P^\star_i, \policy^{(\round)}_i} \sum\nolimits_{h=1}^{\infty} \sum\nolimits_{s' \in \stateset} \gamma^{h-1} (P^{(\round)}_i(s_{h,i}, a_{h,i}, s') - P^\star_i(s_{h,i}, a_{h,i}, s')) V^{P^{(\round)}_i}_{\policy^{(\round)}_i}(s') \nonumber \\
    \leq & \sum\nolimits_{i=1}^N \mathbb{E}_{P^\star_i, \policy^{(\round)}_i} \sum\nolimits_{h=1}^{\infty} \gamma^{h-1} \norm{P^{(\round)}_i(s_{h,i}, a_{h,i}, \cdot) - P^\star_i(s_{h,i}, a_{h,i}, \cdot)}_1 V_{\text{max}} \nonumber \\
    \leq & ~2 \sum\nolimits_{i=1}^N \mathbb{E}_{P^\star_i, \policy^{(\round)}} \sum\nolimits_{h=1}^{\infty} \gamma^{h-1} d^{(\round)}_i(s_{h,i},a_{h,i}) V_{\text{max}} \label{eqn:regret-bound-full-horizon}
\end{align}

Next, we split the term into regret within $H$ horizon and the regret outside of $H$ horizon.
By applying Theorem~\ref{thm:regret-bound-large-horizon} with the assumption (Assumption~\ref{assumption:irreducibility}) of the $H$-step ergodicity $\epsilon$ of MDP associated to arm $i$, we can bound the regret outside of $H$ horizon by the regret at $H$ time step:
\begin{align}
    & \mathbb{E}_{P^\star_i, \policy^{(\round)}} \sum\nolimits_{h=H+1}^{\infty} \gamma^{h-1} d^{(\round)}_i(s_{h,i},a_{h,i}) V_{\text{max}} \nonumber \\
    = & \boldsymbol \sum\nolimits_{h=H+1}^{\infty} \gamma^{h-1} \mathbb{E}_{s_{h,i}, a_{h,i} \sim P^\star_i, \policy^{(\round)}_i} \left[ d^{(\round)}_i(s_{h,i},a_{h,i}) V_{\text{max}} \right] \nonumber \\
    \leq & \boldsymbol \sum\nolimits_{h=H+1}^{\infty} \gamma^{h-1} \frac{1}{\epsilon} \mathbb{E}_{s_{H,i}, a_{H,i} \sim P^\star_i, \policy^{(\round)}_i} \left[ d^{(\round)}_i(s_{h,i},a_{h,i}) V_{\text{max}} \right] \nonumber \\
    =& \frac{\gamma^H}{\epsilon (1-\gamma)} \mathbb{E}_{s_{H,i}, a_{H,i} \sim P^\star_i, \policy^{(\round)}_i} \left[ d^{(\round)}_i(s_{h,i},a_{h,i}) V_{\text{max}} \right] \label{eqn:regret-bound-H-horizon}
\end{align}

Now, we can further bound the contribution of arm $i$ in Equation~\ref{eqn:regret-bound-full-horizon} by substituting the regret after $H$ steps by Equation~\ref{eqn:regret-bound-H-horizon} to get:
\begin{align}
    & \mathbb{E}_{P^\star_i, \policy^{(\round)}_i} \sum_{h=1}^{\infty} \gamma^{h-1}  d^{(\round)}_i(s_{h,i},a_{h,i}) V_{\text{max}} \nonumber \\
    \leq & ~\mathbb{E}_{P^\star_i, \policy^{(\round)}_i} \left ( \sum_{h=1}^{H} \gamma^{h-1}  d^{(\round)}_i(s_{h,i},a_{h,i}) + \frac{\gamma^H}{\delta (1-\gamma)}  d^{(\round)}_i(s_{H,i},a_{H,i}) \right) V_{\text{max}} \nonumber \\
    \leq & ~ \left( 1 + \frac{\gamma^H}{\epsilon (1-\gamma)} \right) \mathbb{E}_{P^\star_i, \policy^{(\round)}_i} \left( \sum_{h=1}^{H}  d^{(\round)}_i(s_{h,i},a_{h,i}) V_{\text{max}} \right) \nonumber \\
    = & ~ \left(1 + \frac{\gamma^H}{\epsilon (1-\gamma)} \right) \sqrt{2 |S| \log(2 |A| N \round)} V_\text{max} \mathbb{E}_{P^\star_i, \policy^{(\round)}_i} \left(\sum_{h=1}^{H}  \frac{1}{\sqrt{ \max \{ 1, N_i^{(\round)}(s,a) \} }} \right) \nonumber \\
    \leq & ~ \left( 1 + \frac{\gamma^H}{\epsilon (1-\gamma)} \right) \sqrt{2 |S| \log(2 |A| N \Round)} V_\text{max} \mathbb{E}_{P^\star_i, \policy^{(\round)}_i} \left( \sum_{s \in \stateset, a \in \Action} \frac{v^{(\round)}_i(s,a)}{\sqrt{\max \{ 1, N_i^{(\round)}(s,a) \}}} \right)
\end{align}
where $v^{(\round)}_i(s,a)$ is a random variable denoting the number of visitations to the pair $(s,a)$ at arm $i$ that the policy $\policy^{(\round)}_i$ visits within $H$ steps under the transition probability $P^\star_i$.

Recall that $\sum\nolimits_{j=1}^{l-1} v^{(j)}_i(s,a) = N_i^{(\round)}(s,a)$. We also know that $0 \leq v^{(j)}_i(s,a) \leq H$. Applying Lemma~\ref{lemma:sum-of-sqrt}, we have:
\begin{align}
    \sum_{\round=1}^\Round \frac{v^{(\round)}_i(s,a)}{\sqrt{\max \{ 1, N_i^{(\round)}(s,a) \}}} \leq \left( \sqrt{H+1} + 1 \right) \sqrt{N_i^{(\round)}(s,a)}
\end{align}

Taking summation over all the $(s,a)$ pairs and applying Jensen inequality give us:
\begin{align}
    & \left( \sqrt{H+1} + 1 \right) \sum\nolimits_{s \in \stateset, a \in \Action} \sqrt{N_i^{(\round)}(s,a)} \nonumber \\
    \leq & \left( \sqrt{H+1} + 1 \right) |S| |A| \sqrt{\frac{\sum\nolimits_{s \in \stateset, a \in \Action} N^{(\round)}_i(s,a)}{|S| |A|}} \nonumber \\
    = & \left( \sqrt{H+1} + 1 \right) \sqrt{|S| |A| \Round \Horizon}
\end{align}
where $\sum\nolimits_{s \in \stateset, a \in \Action} N_i^{(\round)}(s,a) = \Round \Horizon$ is the total number of state-action pairs visited in $\Round$ rounds.

Lastly, using the trivial upper bound $V_\text{max} \leq \frac{R_{\text{max}}}{1-\gamma}$, we can take summation over the regret from all $\Round$ rounds. This give us:
\begin{align}
    & ~\sum\nolimits_{\round=1}^\Round \text{Reg}^{(\round)} \mathds{1}_{\truep \in \boldsymbol B^{(\round)}} \\
    \leq & ~\sum\nolimits_{i=1}^N 2 \left( 1 + \frac{\gamma^\Horizon}{\epsilon (1 - \gamma)} \right) \sqrt{2 |S| \log(2 |A| N \Round)} V_\text{max} \left( \sqrt{\Horizon+1} + 1 \right) \sqrt{|S| |A| \Round \Horizon} \nonumber \\
    \leq & ~O\left(\frac{1}{\epsilon} |S| |A|^{\frac{1}{2}} N H \sqrt{\Round \log \Round} \right)
\end{align}

\paragraph{Combining everything together}
In the first part, we show that $\sum\nolimits_{\round=1}^\Round \text{Reg}^{(\round)} \mathds{1}_{\truep \not\in \boldsymbol B^{(\round)}}$ is upper bounded by $O(\sqrt{\Round})$ for all $\Round \in \N$ with probability $1 - O(\delta)$.
In the second part, we show that $\sum\nolimits_{\round=1}^\Round \text{Reg}^{(\round)} \mathds{1}_{\truep \in \boldsymbol B^{(\round)}} = O(|S| |A|^{\frac{1}{2}} N \sqrt{\Round \log \Round})$.
Therefore, we can conclude that the total regret $\text{Reg}(\Round)$ is upper bounded by $O(|S| |A|^{\frac{1}{2}} N \sqrt{\Round \log \Round})$ for all $\Round \in \N$ with probability $1 - O(\delta)$.
\end{proof}

\subsection{Supplementary Lemma and Theorem}\label{sec:supplementary-lemma}
\begin{assumption}[Ergodic Markov chain]\label{assumption:irreducibility}
We denote $u_h^{P^\star_i, \policy_i}$ to be the state distribution of Markov chain induced by the MDP with transition probability $P^\star_i$ and policy $\policy_i$ after $h$ time steps. We assume $u_h^{P^\star_i, \policy_i}(s) > \epsilon > 0$ for all entry $s \in \stateset$, all arm $i \in [N]$, $h \geq H$, and all policy $\policy_i$. In other words, the state distribution after $H$ steps is universally lower-bounded by $\epsilon > 0$, which we say that the MDP is $H$-step $\epsilon$-ergodic.
\end{assumption}
Assumption~\ref{assumption:irreducibility} can be achieved when both the MDP is ergodic and the horizon~$H$ is large enough.

\begin{theorem}[Regret outside of $H$ steps]\label{thm:regret-bound-large-horizon}
When the Markov chain induced by transition $P^\star_i$ and policy $\policy$ is $H$-step $\epsilon$ ergodic, we have:
\begin{align}
    \mathbb{E}_{s_{h,i}, a_{h,i} \sim P^\star_i, \policy} f(s_{h,i},a_{h,i}) \leq \frac{1}{\epsilon} \mathbb{E}_{s_{H,i}, a_{H,i} \sim P^\star_i, \policy} f(s_{H,i},a_{H,i})
\end{align}
for all non-negative function $f$ and $h \geq H$.
\end{theorem}
\begin{proof}
\begin{align}
    \mathbb{E}_{s_{h,i}, a_{h,i} \sim P^\star_i, \policy} f(s_{h,i},a_{h,i}) = & \sum\nolimits_{s \sim S, a \sim A} \Pr(\policy_i(s) = a) u_h(s)  f(s, a) \nonumber \\
    \leq & \sum\nolimits_{s \sim S, a \sim A} \Pr(\policy_i(s) = a) f(s, a) \nonumber \\
    \leq & \sum\nolimits_{s \sim S, a \sim A} \Pr(\policy_i(s) = a) \frac{u_H(s)}{\epsilon}  f(s, a) \nonumber \\
    = & \frac{1}{\epsilon} \sum\nolimits_{s \sim S, a \sim A} \Pr(\policy_i(s) = a) u_H(s)  f(s, a) \nonumber \\
    = & \frac{1}{\epsilon} E_{s_{H,i}, a_{H,i} \sim P^\star_i, \policy} f(s_{H,i},a_{H,i}) 
\end{align}
\end{proof}

\begin{lemma}\label{lemma:sum-of-sqrt}
For any sequence of numbers $z_1, \cdots, z_T$ with $0 \leq z_j \leq H$ and $Z_\round = \max \{1, \sum\nolimits_{j=1}^\round z_j \}$, we have:
\begin{align}\label{eqn:sum-of-sqrt}
    \sum\nolimits_{\round=1}^\Round \frac{z_\round}{\sqrt{Z_{\round-1}}} \leq \left( \sqrt{H+1} + 1 \right) \sqrt{Z_\Round} \nonumber \\
\end{align}
\end{lemma}
\begin{proof}
Proof by induction. Assume that Equation~\ref{eqn:sum-of-sqrt} holds for $T-1$. We have:
\begin{align}
    \sum\nolimits_{\round=1}^{\Round-1} \frac{z_\round}{\sqrt{Z_{\round-1}}} \leq \left( \sqrt{H+1} + 1 \right) \sqrt{Z_{\Round-1}} \nonumber \\
\end{align}

Adding an additional term $\frac{z_T}{\sqrt{Z_{T-1}}}$, we get:
\begin{align}
    \sum\nolimits_{\round=1}^{\Round-1} \frac{z_\round}{\sqrt{Z_{\round-1}}} + \frac{z_T}{\sqrt{Z_{T-1}}} \leq & \left( \sqrt{H+1} + 1 \right) \sqrt{Z_{\Round-1}} + \frac{z_T}{\sqrt{Z_{T-1}}} \nonumber \\
    = & \sqrt{ \left( \sqrt{H+1} + 1 \right)^2 Z_{T-1} + 2 \left( \sqrt{H+1} + 1 \right) z_T + \frac{z_T^2}{Z_{T-1}}} \nonumber \\
    \leq & \sqrt{ \left( \sqrt{H+1} + 1 \right)^2 Z_{T-1} + 2 \left( \sqrt{H+1} + 1 \right) z_T + H z_T } \nonumber \\
    \leq & \sqrt{\left( \sqrt{H+1} + 1 \right)^2 Z_{T-1} + \left(\sqrt{H+1} + 1 \right)^2 z_T} \nonumber \\
    \leq & \left( \sqrt{H+1} + 1 \right) \sqrt{Z_{T-1} + z_T} \nonumber \\
    = & \left( \sqrt{H+1} + 1 \right) \sqrt{Z_{T}}
\end{align}
which implies the Equation~\ref{eqn:sum-of-sqrt} also holds for $T$.

The initial case with $T=1$ holds trivially. Therefore, by induction, we conclude the proof.
\end{proof}

\subsection{Regret Bound with Unknown Optimal Penalty}
\agnosticRegretBound*

\begin{proof}
The main challenge of an unknown penalty term $\lambda^\star$ is that the optimality of the chosen transition $\boldsymbol P^{(t)}$ and policy $\policy^{(t)}$ does not hold in Theorem~\ref{thm:regret-decomposition-all} due to the misalignment of the penalty $\lambda^{(t)}$ used in solving the optimization in Equation~\eqref{eqn:optimistic-whittle} and the penalty $\lambda^\star$ used in computing the regret.

The optimality of $\lambda^{(t)}$ (minimizing $U^{\boldsymbol P, \lambda}_{\policy}$) and the optimality of $\boldsymbol P^{(t)}, \pi^{(t)}$ (maximizing $U^{\boldsymbol P, \lambda}_{\policy}$) are given by:
\begin{align}
    \lambda^{(t)}, \boldsymbol P^{(t)}, \policy^{(t)} = \arg\min\limits_{\lambda}\max\limits_{\boldsymbol P, \policy} U^{\boldsymbol P, \lambda}_{\policy} \nonumber
\end{align}
which give us, respectively:
\begin{align}
    U^{\boldsymbol P^{(t)}, \lambda^{(t)}}_{\policy^{(t)}} \leq U^{\boldsymbol P^{(t)}, \lambda^\star}_{\policy^{(t)}}, \quad U^{\truep, \lambda^{(t)}}_{\optpolicy} \leq U^{\boldsymbol P^{(t)}, \lambda^{(t)}}_{\policy^{(t)}} \label{eqn:cute-ineq1} 
\end{align}

Similarly, the optimality of $\lambda^\star$ can be written as:
\begin{align}
    \lambda^\star= \arg\min\limits_{\lambda} U^{\truep, \lambda}_{\optpolicy} \nonumber
\end{align}
which gives us 
\begin{align}
    U^{\truep, \lambda^\star}_{\optpolicy} \leq U^{\truep, \lambda^{(t)}}_{\optpolicy} \label{eqn:cute-ineq2}
\end{align}

Combining Inequality~\ref{eqn:cute-ineq1} and Inequality~\ref{eqn:cute-ineq2}, we can bound:
\begin{align}
    U^{\truep, \lambda^\star}_{\optpolicy} \leq U^{\truep, \lambda^{(t)}}_{\optpolicy} \leq U^{\boldsymbol P^{(t)}, \lambda^{(t)}}_{\policy^{(t)}} \leq U^{\boldsymbol P^{(t)}, \lambda^\star}_{\policy^{(t)}} \nonumber
\end{align}
This implies that:
\begin{align}
     \text{Reg}_{\lambda^\star}^{(t)} = U^{\truep, \lambda^\star}_{\optpolicy} - U^{\truep, \lambda^\star}_{\policy^{(t)}} \leq U^{\boldsymbol P^{(t)}, \lambda^\star}_{\policy^{(t)}} - U^{\truep, \lambda^\star}_{\policy^{(t)}}
\end{align}
which is exactly the same result as shown in Equation~\ref{eqn:regret-decomposition-all}.
The rest of the proof follows the same argument of Theorem~\ref{thm:regret-decomposition} and Theorem~\ref{thm:regret-bound}, which concludes the proof.
\end{proof}

\subsection{Choice of Horizon and Ergodicity Constant $\epsilon$}
For a given Markov chain, we need $H$ to be sufficiently large to ensure the probability of visiting any state after $H$ steps is at least a positive constant $\epsilon > 0$. The choice of $H$ depends on the MDP; we elaborate below how to select $H$ and $\epsilon$.

We follow a similar analysis of Markov chain convergence from Chapter 10 in \cite{spielman2007spectral} by defining:

$$\omega_2 = \max_{\pi \in \Pi} \sigma_2(P_\pi)$$

where $\sigma_2(P)$ is the magnitude of the second largest eigenvalue of the random walk matrix $P_\pi$ induced by the policy $\pi$. In practice, $\omega_2$ can be upper bounded by 1 if the MDP satisfies some properties, e.g., laziness of the Markov chain induced from the MDP (Chapter 10.2 in \cite{spielman2007spectral}).

Let $v$ be the corresponding stationary distribution of the random walk matrix $P_\pi$ with the policy $\pi$ that maximizes the second largest eigenvalue. We know that $v$ is strictly positive by ergodicity. When $\sigma_2 < 1$, we can write $r = \min_i v_i > 0$ and choose $\epsilon = \frac{1}{2} r > 0$.

Let $w$ be an arbitrary initial distribution. By applying Theorem 10.4.1 from~\cite{spielman2007spectral} (the directed graph version), for every $t > H = \log_{\omega_2} (\frac{1}{2} r^{3/2}) = \log_{\omega_2} (\sqrt{2} \epsilon^{3/2})$, we have:

$$ | v - P_\pi^t w |_1 \leq \sqrt{\frac{1}{\min_i v_i}} \omega_2^t \leq \frac{r}{2}$$
which implies that the minimum value of $P_\pi^t w$ and the minimum value of $v$, i.e., $r$, differ by at most $\frac{r}{2}$. This implies that the minimum value of $P_\pi^t w$ is at least $\frac{r}{2} = \epsilon$ for any initial distribution $w$. This choice of $\epsilon$ and $H$ satisfies our requirement mentioned in Appendix~\ref{sec:supplementary-lemma}.

\section{Experiment Details}
\label{sec:experiment-addl-details}

\subsection{Whittle Index Implementation Speedups}
\label{sec:implementation-details}

We introduce a number of implementation-level improvements to speed up the computation of Whittle indices. To our knowledge these approaches are novel for Whittle index computation.

\paragraph{Early termination} The key insight is that the Whittle index threshold policy will pull the arms with the $K$ largest Whittle indices. As we compute Whittle indices for each of the $N$ arms, after we have computed the first $K$ Whittle indices, any future arm selected would have to have Whittle index at least as high as the $K$-th largest seen so far in order to be pulled. Let us notate the $K$-th largest value seen so far as $\texttt{top-k}$.

Whittle indices are computed using a binary search procedure \citep{qian2016restless}, which at each iteration tracks the upper bound $\overline{\lambda}$ and lower bound $\underline{\lambda}$ of the index. Once the upper bound falls below that of the minimum value of the $K$ largest indices so far $\overline{\lambda} < \texttt{top-k}$, then we can terminate the binary search procedure as we are guaranteed that we would not act on that arm anyways. 
We implement the tracking of the $K$ largest indices so far with a priority queue.

Similarly, we implement early termination to solve the bilinear programs $\eqref{eqn:optimistic-whittle}$ and $\eqref{eqn:optimistic-whittle-heuristic}$ as callbacks in the Gurobi solver, in which we check the value of the current objective bound. 

\paragraph{Memoization} We memoize every Whittle index result computed throughout execution to track the index resulting from each pair of probabilities $P_i$ and current state $s_i$ as we perform calculations for each arm~$i$. We implement this memoizer as a dictionary where the key is a tuple $(P_i, s_i)$ with $P_i$ recorded to four decimal places. 

To implement the bilinear programs $\eqref{eqn:optimistic-whittle}$ and $\eqref{eqn:optimistic-whittle-heuristic}$, we similarly memoize using the lower confidence bound (LCB) and upper confidence bound (UCB) that comprise the space $\boldsymbol B_i^{(\round)}$. 




\subsection{Synthetic Data}

The synthetic datasets are created by generating transition probabilities $P^i_{s, a, s'}$ sampled uniformly at random from the interval $[0, 1]$ for each arm~$i$, starting state~$s$, action~$a$, and next state~$s'$. Specifically we select transition probabilities for the probability of transitioning to a good state $P^i_{s, a, s'=1}$, then set $P^i_{s, a, s'=0} = 1 - P^i_{s, a, s'=1}$.

To ensure the validity constraints that acting is always helpful and starting in the good state is always helpful, we apply the following: for all arms~$i \in [N]$: 
\begin{itemize}
    \item \emph{Acting is always helpful:} If this requirement is violated with $P^i_{s, a=1, 1} < P^i_{s, a=0, 1}$, then $P^i_{s, a=0, 1} = P^i_{s, a=1, 1} \times \eta$ where $\eta$ is uniform noise sampled between $[0, 1]$.
    \item \emph{Starting from good state is always helpful:} If this requirement is violated with $P^i_{s=1, a, 1} < P^i_{s=0, a, 1}$, then $P^i_{s=0, a, 1} = P^i_{s=1, a, 1} \times \eta$ where $\eta$ is uniform noise sampled between $[0, 1]$.
\end{itemize}

The \emph{thin margin} dataset is created by mirroring the procedure described above but then constraining the probability of transitioning to a good state $P^i_{s, a, s'=1}$ to the interval $[0.2, 0.4]$. Thus the probabilities of transitioning to the bad state $P^i_{s, a, s'=0}$ are all between $[0.6, 0.8]$.

\subsection{Acting in Low-Budget Settings}

The potential impact of effectively allocating one resource is greater in low-budget settings. As one example, the ARMMAN setting from our experiments helps distribute a small number of healthcare workers across a group of pregnant women for preventative health care. We study real data from ARMMAN to show that the performance gap between approaches is wider in low-budget settings.

Using one actual instance from ARMMAN, we consider distributing healthcare workers across  mothers (arms). Using the true transition probabilities, we calculate the (sorted) Whittle indices of an optimal policy as: 0.42, 0.39, 0.28, 0.23, 0.19, 0.11, 0.07, 0.

In the table below, we first show the expected reward of the optimal action and a random action (baseline) as we increase budget  in the ARMMAN problem. We then calculate the difference in reward between the optimal action and random action for each budget level, normalized per worker. It is clear that the potential impact over the baseline of effectively allocating one worker is greater in low budget settings.

\begin{table}[!ht]
\centering
\begin{tabular}{lrrr}
\toprule
& \multicolumn{2}{c}{Reward} & Reward gap per worker \\
\cmidrule{2-3}
$K$ & Optimal & Random & $(\text{Opt} - \text{Random}) / K$ \\
\midrule
1 & 0.42 & 0.211 & 0.209 \\
2 & 0.81 & 0.423 & 0.194 \\
3 & 1.09 & 0.634 & 0.152 \\
4 & 1.32 & 0.845 & 0.119 \\
5 & 1.51 & 1.056 & 0.091 \\
6 & 1.62 & 1.268 & 0.059 \\
7 & 1.69 & 1.479 & 0.030 \\
8 & 1.69 & 1.690 & 0.000 \\
\bottomrule
\end{tabular}
\caption{Reward contribution from each worker}
\end{table}

\subsection{Computation Infrastructure}
All results are averaged over 30 random seeds. Experiments were executed on a cluster running CentOS with Intel(R) Xeon(R) CPU E5-2683 v4 @ 2.1 GHz with 8GB of RAM using Python 3.9.12. The bilinear program solved using Gurobi optimizer 9.5.1. 
